%% file: main_submission.tex
\pdfoutput=1

\documentclass[11pt,twoside]{article}

\usepackage{amsmath}
\usepackage{amsthm}
\usepackage{amssymb}
\usepackage{amsfonts}
\usepackage[top = 1 in, bottom = 1.2 in, left = 1 in, right = 1 in]{geometry}
\usepackage{color}
\usepackage{mathrsfs}
\usepackage[normalem]{ulem}
\usepackage{longtable}
\usepackage{hyperref}
\usepackage{cleveref}
\usepackage{dsfont}
\usepackage{graphicx}
\usepackage{caption}
\usepackage{subcaption}
\usepackage{float}
\usepackage{mathtools}
\usepackage{etoolbox}
\usepackage{thmtools}
\usepackage{enumitem}
\usepackage[textsize=tiny]{todonotes}
\usepackage{multirow}

\usepackage{esvect}

\usepackage[utf8]{inputenc}

\newtheorem{theorem}{Theorem}
\newtheorem{lemma}{Lemma}

\newtheorem{proposition}{Proposition}

\newtheorem{assumption}{Assumption}
\declaretheoremstyle[headfont=\bf,bodyfont=\normalfont]{ex}
\declaretheorem[style=ex]{example}
\declaretheoremstyle[bodyfont=\normalfont]{rm}

\newcommand{\normmm}{{\vert\kern-0.25ex\vert\kern-0.25ex\vert}}
\newcommand{\bignormmm}{{\big\vert\kern-0.25ex\big\vert\kern-0.25ex\big\vert}}
\newcommand{\Bignormmm}{{\Big\vert\kern-0.25ex\Big\vert\kern-0.25ex\Big\vert}}

\makeatletter
\long\def\@makecaption#1#2{
\vskip 0.8ex
\setbox\@tempboxa\hbox{\small {\bf #1:} #2}
\parindent 1.5em  
\dimen0=\hsize
\advance\dimen0 by -3em
\ifdim \wd\@tempboxa >\dimen0
\hbox to \hsize{
\parindent 0em
\hfil 
\parbox{\dimen0}{\def\baselinestretch{0.96}\small
{\bf #1.} {#2}
} 
\hfil}
\else \hbox to \hsize{\hfil \box\@tempboxa \hfil}
\fi
}
\makeatother

\usepackage{hyperref}

\input{macros}

\usepackage{algorithm,algorithmicx,algpseudocode}
\usepackage{xspace}
\newcommand\AlgoName{\texttt{K-GT-Minimax}\xspace}
\usepackage{booktabs} 
\usepackage{colortbl}
\definecolor{LightCyan}{RGB}{191, 239, 255} 

\begin{document}

\begin{center}
{\bf \LARGE
Fast Decentralized Gradient Tracking for Federated
\\
Minimax Optimization with Local Updates
} \\

\vspace{1em}
{\large{
\begin{tabular}{ccc}
Chris Junchi Li$^\diamond$\\
\end{tabular}

\medskip

\begin{tabular}{c}
Department of Electrical Engineering and Computer Sciences$^\diamond$\\
University of California, Berkeley
\end{tabular}

}}
\vspace{.6em}
\today
\end{center}

\begin{center}
{\bf Abstract} \\ \vspace{.6em}
\begin{minipage}{0.9\linewidth}
{\small ~~~~ 
Federated learning (FL) for minimax optimization has emerged as a powerful paradigm for training models across distributed nodes/clients while preserving data privacy and model robustness on data heterogeneity.
In this work, we delve into the decentralized implementation of federated minimax optimization by proposing \AlgoName, a novel decentralized minimax optimization algorithm that combines local updates and gradient tracking techniques.
Our analysis showcases the algorithm's communication efficiency and convergence rate for nonconvex-strongly-concave (NC-SC) minimax optimization, demonstrating a superior convergence rate compared to existing methods.
\AlgoName's ability to handle data heterogeneity and ensure robustness underscores its significance in advancing federated learning research and applications.
}
\end{minipage}
\end{center}



\section{Introduction}\label{sec_intro}
In this paper, we delve into the realm of federated minimax optimization, focusing on a decentralized network comprising $n$ agents tasked with optimizing the objective function:
\begin{equation}\label{eq1}
\min _{\mathbf{x} \in \mathbb{R}^{d_{x}}} \max _{\mathbf{y} \in \mathbb{R}^{d_{y}}} f(\mathbf{x}, \mathbf{y}) \equiv \frac{1}{n} \sum_{i=1}^{n} f_{i}(\mathbf{x}, \mathbf{y})
\end{equation}
Here, $f_{i}(\mathbf{x}, \mathbf{y})$ represents the local function associated with client $i \in \mathcal{V}=[n]\equiv\{1, \ldots, n\}$, where $\mathcal{D}_{i}$ denotes the distribution of the data.
Our focus lies on the challenging domain of \emph{nonconvex-strongly-concave (NC-SC) minimax problems}, particularly in the context of federated learning setups, where $f(\mathbf{x}, \mathbf{y})$ admits $\mu$ strong concavity with respect to $\mathbf{y}$, while each local function $f_{i}(\mathbf{x}, \mathbf{y})$ is in expectation form indexed by random vector $\xi_{i}$.
Problem~\eqref{eq1} finds rich applications in adversarial training, distributionally robust optimization, reinforcement learning, AUC maximization, and learning with non-decomposable loss~\cite{goodfellow2020generative,zhang2021complexity,kairouz2021advances,hsieh2020non}.

Decentralized minimax optimization has emerged as a crucial area of research in machine learning, addressing complex optimization challenges within decentralized networks.
In contrast to centralized methods, the decentralized approach facilitates efficient collaboration among agents while mitigating communication bottlenecks.
Recent advancements in nonconvex minimization and minimax optimization have led to the exploration of achieving stationary points in the primal function $\Phi(\mathbf{x}) \equiv \max _{\mathbf{y} \in \mathbb{R}^{d_{y}}} f(\mathbf{x}, \mathbf{y})$, enhancing scalability and model robustness in decentralized optimization algorithms tailored for federated learning environments.
Significant advancements have been made in addressing the unique challenges of federated minimax optimization, such as data heterogeneity, model robustness, and communication challenges.

The significance of decentralized minimax optimization spans diverse machine learning applications, showcasing its versatility and potential in improving learning efficiency and scalability.
Noteworthy advancements such as variants of decentralized stochastic gradient methods have demonstrated promising results in both online and offline scenarios~\cite{chen2024efficient,li2020communication,koloskova2021improved}.
These developments underline ongoing efforts to address communication bottlenecks and enhance collaboration efficiency in decentralized optimization frameworks, paving the way for novel techniques that navigate challenging nonconvex optimization landscapes in distributed environments.

\paragraph{Related Work.}
Decentralized minimax optimization, particularly in the context of federated learning setups, has gained traction for applications like adversarial training, distributionally robust optimization, and reinforcement learning~\cite{goodfellow2020generative,zhang2021complexity,kairouz2021advances}.
Recent theoretical-front research focuses on $\epsilon$-stationary points in nonconvex-strongly-concave (NC-SC) minimax problems and the importance of distributed optimization techniques for large-scale machine learning~\cite{li2020communication}.
While centralized minimax optimization remains prominent, particularly in adversarial training and GANs, leveraging techniques for specific stochastic gradient complexities~\cite{lin2020gradient,luo2020stochastic,zhang2021complexity}, decentralized minimax optimization in federated learning from convex-concave to nonconvex-(non)concave objectives still face challenges regarding decentralization and effective gradient tracking~\cite{sun2022communication,sharma2022federated}.
Various decentralized optimization algorithms like SGDA, SREDA, GT-DA, and GT-GDA offer trade-offs in computational and communication complexities~\cite{rogozin2024decentralized,liu2023precision,xu2024decentralized,zhang2019decentralized}.
Building upon variance-reduced minimax optimization, DREAM shows superior performance in decentralized minimax optimization, yet a comprehensive understanding of NC-SC minimax problems remains an active research area~\cite{chen2024efficient}.

Addressing data heterogeneity is crucial in federated learning, prompting studies like FedPAGE, Federated Bose-Einstein Optimization, and Federated Learning with Decentralized Gradient Tracking~\cite{zhou2023fedpage,wang2023bose,koloskova2021improved}.
However, these approaches often lack decentralization and may make restrictive gradient assumptions.
Similarly, decentralized algorithms like K-GT~\cite{liu2024decentralized} and LU-GT~\cite{nguyen2023performance} primarily target minimization tasks, limiting their applicability in federated learning.
Notably, algorithms like K-GT have shown promise in improving communication efficiency and robustness in federated minimax optimization tasks.
It is a novel decentralized tracking mechanism that improves communication efficiency in Gradient Tracking algorithms, overcoming data heterogeneity between clients, and demonstrating model robustness in solving non-convex optimization problems, including neural network training tasks.
Our work also heavily uses gradient tracking techniques, whereas our algorithm name~\AlgoName originates from K-GT proposed by~\cite{liu2024decentralized} for decentralized single-agent optimization, although it should not be viewed as a straightforward generalization.

\paragraph{Our Contribution.}
This paper introduces \AlgoName, a novel decentralized minimax optimization algorithm designed specifically for federated learning environments.
By combining gradient tracking and local updates, \AlgoName addresses challenges related to data heterogeneity, model robustness, and communication efficiency.
Furthermore, our algorithm demonstrates superior performance in terms of convergence rates, scalability, and stochastic gradient complexity compared to existing methods.

\paragraph{Notations.}
Throughout this paper, we use $\|\cdot\|$ for Euclidean norms, $\mathbf{I}$ for identity matrices, and $\mathbf{1}_{n}$ for a vector of all ones.
Aggregated variables are denoted by $\mathbf{x}$, $\mathbf{y}$, representing agents' local variables and gradients.
We define $\mathcal{V} = [n] \equiv \{1,\ldots n\}$ as the set of agents, with communication edges denoted by $\mathcal{E} \subseteq \mathcal{V}\times\mathcal{V}$.
Additionally, we introduce $d$ by $n$ real matrix $\overline{\mathbf{X}}=[\overline{\mathbf{x}}, \ldots, \overline{\mathbf{x}}]=\mathbf{X} \mathbf{J}$ where $\mathbf{J} \equiv \frac{1}{n} \mathbf{1}_{n} \mathbf{1}_{n}^{\top}$ and $\delta_{i j}$ represents the Kronecker delta with $\delta_{i j}=1$ if $i=j$ and 0 otherwise.
Other notations will be introduced at their first appearances.

\section{Settings and Main Results}\label{sec_settings}
We propose Algorithm~\ref{algo1} for solving this problem in a distributed manner.
To prepare for our main result, we introduce the following assumptions.

\begin{assumption}[Lower Bound of $\Phi(\cdot)$]\label{assu1}
The function $\Phi(\mathbf{x}) \equiv \max _{\mathbf{y} \in \mathbb{R}^{d_{y}}} f(\mathbf{x}, \mathbf{y})$ is lower bounded, that is
$$
\Phi^{*}=\inf _{\mathbf{x}} \Phi(\mathbf{x})>-\infty
$$
\end{assumption}

\begin{assumption}[Smoothness and Strong Concavity]\label{assu2}
For each $i \in[n]$ let each local objective $f_{i}$ : $\mathbb{R}^{d_{x}} \times \mathbb{R}^{d_{y}} \rightarrow \mathbb{R}$ be twice differentiable and L-smooth for some constant $L>0$, i.e.~for all $\mathbf{x}, \mathbf{x}^{\prime} \in \mathbb{R}^{d_{x}}$, $\mathbf{y}, \mathbf{y}^{\prime} \in \mathbb{R}^{d_{y}}$
$$
\left\|\nabla f_{i}(\mathbf{x}, \mathbf{y})-\nabla f_{i}(\mathbf{x}^{\prime}, \mathbf{y}^{\prime})\right\|^{2}
\le
L^{2}\left(\|\mathbf{x}-\mathbf{x}^{\prime}\|^{2}+\|\mathbf{y}-\mathbf{y}^{\prime}\|^{2}\right)
$$
Assume further $f_{i}(\mathbf{x}, \cdot)$ is $\mu$-strongly concave for some shared $\mu>0$ across all $x \in \mathbb{R}^{p}$, i.e.~for all $\mathbf{x} \in \mathbb{R}^{d_{x}}$, $\mathbf{y}, \mathbf{y}^{\prime} \in \mathbb{R}^{d_{y}}$
$$
f_{i}(\mathbf{x}, \mathbf{y}^{\prime})
\le
f_{i}(\mathbf{x}, \mathbf{y})+\nabla_{\mathbf{y}} f_{i}(\mathbf{x}, \mathbf{y})^{\top}(\mathbf{y}^{\prime}-\mathbf{y})-\frac{\mu}{2}\left\|\mathbf{y}^{\prime}-\mathbf{y}\right\|^{2}
$$
Call $\kappa \equiv L / \mu$ the \emph{condition number} of Problem~\eqref{eq1}.
\end{assumption}

\begin{algorithm}[!tb]
\centering
\begin{algorithmic}[1]
\State\textbf{Initialize:}
Communication round $T$; Number of local steps $K$; Local stepsize $\eta_c^{\mathbf{x}}, \eta_c^{\mathbf{y}}$; Communication stepsizes $\eta_s^{\mathbf{x}}, \eta_s^{\mathbf{y}}$; Mixing matrix $\mathbf{W}=\left(w_{i j}\right)_{n \times n}$; $\forall i, j \in[n], \mathbf{x}_i^{(0)}=\mathbf{x}_j^{(0)}$, $\mathbf{y}_i^{(0)}= \mathbf{y}_j^{(0)}$, $\mathbf{c}_i^{\mathbf{x},(0)}=-\nabla_{\mathbf{x}} F_i(\mathbf{x}^{(0)}, \mathbf{y}^{(0)} ; \xi_i)+\frac{1}{n} \sum_{j=1}^n \nabla_{\mathbf{x}} F_j(\mathbf{x}^{(0)}, \mathbf{y}^{(0)} ; \xi_j)$, $\mathbf{c}_i^{\mathbf{y},(0)}=-\nabla_{\mathbf{y}} F_i(\mathbf{x}^{(0)}, \mathbf{y}^{(0)} ; \xi_i)+ \frac{1}{n} \sum_{j=1}^n \nabla_{\mathbf{y}} F_j(\mathbf{x}^{(0)}, \mathbf{y}^{(0)} ; \xi_j)
$
\For{clinet $i \in[n]$ parallel}
\For{\textbf{communication:} $t \leftarrow 0$ to $T-1$}
\For{\textbf{local step:} $k \leftarrow 0$ to $K-1$}
\State
$
\mathbf{x}_i^{(t)+k+1}=\mathbf{x}_i^{(t)+k}-\eta_c^{\mathbf{x}}\left(\nabla F_i(\mathbf{x}_i^{(t)+k}, \mathbf{y}_i^{(t)+k} ; \xi_i^{(t)+k})+\mathbf{c}_i^{\mathbf{x},(t)}\right)
$
\State
$
\mathbf{y}_i^{(t)+k+1}=\mathbf{y}_i^{(t)+k}+\eta_c^{\mathbf{y}}\left(\nabla F_i(\mathbf{x}_i^{(t)+k}, \mathbf{y}_i^{(t)+k} ; \xi_i^{(t)+k})+\mathbf{c}_i^{\mathbf{y},(t)}\right)
$
\Comment{variable update}
\EndFor
\State
$
\mathbf{c}_i^{\mathbf{x},(t+1)}=\mathbf{c}_i^{\mathbf{x},(t)}+\frac{1}{K \eta_c^{\mathbf{x}}} \sum_{j=1}^n\left(\delta_{i j}-w_{i j}\right)[\mathbf{x}_j^{(t)+K}-\mathbf{x}_j^{(t)}]
$
\State
$\mathbf{c}_i^{\mathbf{y},(t+1)}=\mathbf{c}_i^{\mathbf{y},(t)}-\frac{1}{K \eta_c^{\mathbf{y}}} \sum_{j=1}^n\left(\delta_{i j}-w_{i j}\right)[\mathbf{y}_j^{(t)+K}-\mathbf{y}_j^{(t)}]
$
\Comment{tracking variable update}
\State
$
\mathbf{x}_i^{(t+1)}=\sum_{j=1}^n w_{i j}\left(\mathbf{x}_j^{(t)}+\eta_s^{\mathbf{x}}[\mathbf{x}_i^{(t)+K}-\mathbf{x}_i^{(t)}]\right)
$
\State
$
\mathbf{y}_i^{(t+1)}=\sum_{j=1}^n w_{i j}\left(\mathbf{y}_j^{(t)}+\eta_s^{\mathbf{y}}[\mathbf{y}_i^{(t)+K}-\mathbf{y}_i^{(t)}]\right)
$
\Comment{model parameter update}
\EndFor
\State\textbf{Output:}
$
\mathbf{x}_{\text {out }}=\overline{\mathbf{x}}^{(\mathcal{T})} \equiv \frac{1}{n} \sum_{i=1}^n \mathbf{x}_i^{(\mathcal{T})}
$ for randomized $\mathcal{T} \in\{0,1, \cdots, T-1\}$
\EndFor
\end{algorithmic}
\caption{Graident Tracking for Minimax Optimization (\AlgoName)}
\label{algo1}
\end{algorithm}

\begin{assumption}[Unbiaseness and Bounded Variance]\label{assu3}
We assume that the stochastic gradients are unbiased and have bounded variance
$$
\mathbb{E}\left[\nabla F_{i}\left(\mathbf{x}, \mathbf{y} ; \xi_{i}\right)\right]=\nabla f_{i}(\mathbf{x}, \mathbf{y})
\qquad
\mathbb{E}\left\|\nabla F_{i}\left(\mathbf{x}, \mathbf{y} ; \xi_{i}\right)-\nabla f_{i}(\mathbf{x}, \mathbf{y})\right\|^{2} \leq \sigma^{2}
$$
\end{assumption}

\begin{assumption}[$(1-p)$-Mixing]\label{assu4}
The $n \times n$ mixing matrix $\mathbf{W}$ is symmetric, element-wise nonnegative,%
\footnote{$W_{ij} > 0$ if and only if $i$ and $j$ are connected.}
doubly stochastic in the sense that $\mathbf{W} \mathbf{1}_{n}=\mathbf{W}^{\top} \mathbf{1}_{n}=1$, and there exists a constant $p \in[0,1]$ such that for any $\mathbf{X} \in \mathbb{R}^{d \times n}$
$$
\|\mathbf{X} \mathbf{W}-\overline{\mathbf{X}}\|_{F}^{2} \leq(1-p)\|\mathbf{X}-\overline{\mathbf{X}}\|_{F}^{2}
$$
\end{assumption}

We are ready to present our main theorem:

\begin{theorem}[Algorithm Complexity of \AlgoName]\label{theo1}
Let Assumptions~\ref{assu1},~\ref{assu2},~\ref{assu3} and~\ref{assu4} hold.
There exists a global constant $v>0$ such that, running~\AlgoName as in Algorithm~\ref{algo1} with stepsizes choice $\eta_{c}^{\mathbf{y}}=\frac{p}{300 v \cdot \kappa K L}$, $\eta_{c}^{\mathbf{x}}=\frac{\eta_{c}^{y}}{\kappa^{2}}$ and $\eta_{s}^{\mathbf{x}}=\eta_{s}^{\mathbf{y}}=v \cdot p$ gives $\mathbb{E}\left\|\nabla \Phi\left(\overline{\mathbf{x}}^{(\mathcal{T})}\right)\right\|^{2} \leq \varepsilon^{2}$ for $T$ communication rounds, each with $K$ local updates, where
\begin{equation}\label{eq_theo1}
T=O\left(\frac{\sigma^{2}}{n K} \frac{1}{\varepsilon^{4}}+\frac{\sigma}{p^{2} \sqrt{K}} \frac{1}{\varepsilon^{3}}+\frac{\kappa^{3}}{p^{2}} \frac{1}{\varepsilon^{2}}\right) \cdot L \mathscr{H}_{0}
\qquad
K=\Omega\left(\left(1+\frac{\kappa}{\sqrt{n p}}\right) \frac{\sigma}{\varepsilon}\right)
\end{equation}
\end{theorem}

To interpret the bound in~\eqref{eq_theo1}, note we have when $\Phi\left(\mathbf{x}_{0}\right)-\Phi^{*}=O(1)$ that $\mathscr{H}_{0}= O\left(1+\frac{1}{\mu^{2} K \kappa p}\right)$.
Further balancing $T$ and $K$ gives
$$
T=O\left(\frac{\kappa^{3}}{p^{2} \varepsilon^{2}}\right) L \mathscr{H}_{0}
\qquad
K=O\left(\frac{p^{2} \sigma^{2}}{\kappa^{2} n \varepsilon^{2}} \vee\left(1+\frac{\kappa}{\sqrt{n p}}\right) \frac{\sigma}{\varepsilon}\right)
$$
The theorem above demonstrates how the convergence rate is affected by the accuracy parameter $\varepsilon>0$ which $\varepsilon$ vanishes to zero, \AlgoName converges to an $\varepsilon$-stationary point within $T=O\left(1 / \varepsilon^{2}\right)$ communication rounds, each round comprising $K=O\left(1 / \varepsilon^{2}\right)$ local updates.
Such convergence rates incorporate and balance heterogeneity, local updates and model robustness.
We list a table of comparison in Table~\ref{table1}.

\begin{table}[!tb]
\centering
\begin{tabular}{c|cc|ccc}
\toprule
Algorithm & Query & Communication & Decentralized & LU & DH \\
\midrule
MLSGDA~\cite{sharma2022federated} & $\kappa^{4} / n \epsilon^{4}$ & $\kappa^{3} / \epsilon^{3}$ & $\times$ & $\checkmark$ & $\times$ \\
SAGDA~\cite{yang2022sagda} & $\kappa^{4} / n \epsilon^{4}$ & $\kappa^{2} / \epsilon^{2}$ & $\times$ & $\checkmark$ & $\times$ \\
Fed-Norm-SGDA~\cite{sharma2023federated} & $\kappa^{4} / n \epsilon^{4}$ & $\kappa^{2} / \epsilon^{2}$ & $\times$ & $\checkmark$ & $\times$ \\
DM-HSGD~\cite{xian2021faster} & $\kappa^{3} / \epsilon^{3}$ & $\kappa^{3} / \epsilon^{3}$ & $\checkmark$ & $\times$ & $\checkmark$ \\
DREAM~\cite{chen2024efficient} & $\kappa^{3} / \epsilon^{3}$ & $\kappa^{2} / \epsilon^{2}$ & $\checkmark$ & $\times$ & $\checkmark$ \\
\rowcolor{LightCyan}
\AlgoName (This work) & $\kappa / n \epsilon^{4}$ & $\kappa^{3} / \epsilon^{2}$ & $\checkmark$ & $\checkmark$ & $\checkmark$ \\
\bottomrule
\end{tabular}
\caption{Comparison of \AlgoName with related algorithms for decentralized minimax optimization, highlighting stochastic gradient oracle complexity, communication rounds, decentralization, local updates (LU) and data heterogeneity (DH) robustness.}
\label{table1}
\end{table}

\section{Proof of Main Theorem}\label{sec_proof,theo_main}
To prepare for the proof, we introduce some notions.

\begin{itemize}
\item
The \emph{client variance} quantifies how much variables $\mathbf{x}$ and $\mathbf{y}$ deviate from its average model across global steps:
$$
\Xi_{t}^{\mathbf{x}} \equiv \frac{1}{n} \sum_{i=1}^{n} \mathbb{E}\left\|\mathbf{x}_{i}^{(t)}-\overline{\mathbf{x}}^{(t)}\right\|^{2}
$$
and
$$
\Xi_{t}^{\mathbf{y}} \equiv \frac{1}{n} \sum_{i=1}^{n} \mathbb{E}\left\|\mathbf{y}_{i}^{(t)}-\overline{\mathbf{y}}^{(t)}\right\|^{2}
$$

\item
The \emph{client drift} quantifies how variables $\mathbf{x}$ and $\mathbf{y}$ deviate from its averaged model across local steps:
$$
e_{k, t}^{\mathbf{x}} \equiv \frac{1}{n} \sum_{i=1}^{n} \mathbb{E}\left\|\mathbf{x}_{i}^{(t)+k}-\overline{\mathbf{x}}^{(t)}\right\|^{2}
\qquad
\mathcal{E}_{t}^{\mathbf{x}} \equiv \sum_{k=0}^{K-1} e_{k, t}^{\mathbf{x}}
$$
and
$$
e_{k, t}^{\mathbf{y}} \equiv \frac{1}{n} \sum_{j=1}^{n} \mathbb{E}\left\|\mathbf{y}_{j}^{(t)+k}-\overline{\mathbf{y}}^{(t)}\right\|^{2}
\qquad
\mathcal{E}_{t}^{\mathbf{y}} \equiv \sum_{k=0}^{K-1} e_{k, t}^{\mathbf{y}}
$$
where $\mathcal{E}_{t}^{\mathbf{x}}$ characterizes the accumulation of local steps for variable $\mathbf{x}$ and analogously $\mathcal{E}_{t}^{\mathbf{y}}$ for $\mathbf{y}$.

\item
The \emph{quality of correction} that assesses the accuracy of the gradient correction across local steps, aiming to closely align local updates with global updates:
$$
\gamma_{t}^{\mathbf{x}}=\frac{1}{n L^{2}} \mathbb{E}\left\|\mathbf{C}^{x,(t)}+\nabla_{\mathbf{x}} f(\overline{\mathbf{x}}^{(t)}, \overline{\mathbf{y}}^{(t)})-\nabla_{\mathbf{x}} f(\overline{\mathbf{x}}^{(t)}, \overline{\mathbf{y}}^{(t)}) \mathbf{J}\right\|_{F}^{2}
$$
and
$$
\gamma_{t}^{\mathbf{y}}=\frac{1}{n L^{2}} \mathbb{E}\left\|\mathbf{C}^{y,(t)}+\nabla_{\mathbf{y}} f(\overline{\mathbf{x}}^{(t)}, \overline{\mathbf{y}}^{(t)})-\nabla_{\mathbf{y}} f(\overline{\mathbf{x}}^{(t)}, \overline{\mathbf{y}}^{(t)}) \mathbf{J}\right\|_{F}^{2}
$$

\item
The \emph{consensus distance for variable $\mathbf{y}$} which quantifies the difference between the averaged value $\overline{\mathbf{y}}^{(t)}$ and the optimal value (when $\mathbf{x}$ equals its average $\overline{\mathbf{x}}^{(t)}$ ) of $\hat{\mathbf{y}}^{(t)}=\arg \max _{\mathbf{v} \in \mathbb{R}^{d y}} f\left(\overline{\mathbf{x}}^{(t)}, \mathbf{y}\right)$ :
$$
\varepsilon_{t}=\left\|\overline{\mathbf{y}}^{(t)}-\hat{\mathbf{y}}^{(t)}\right\|^{2}
$$
\end{itemize}

With these notions we present recursion bounds for the aforementioned client variance and client drift, along with the quality of correction, for both variables $\mathbf{x}$ and $\mathbf{y}$.
We finally discuss the consensus distance specifically for variable $\mathbf{y}$.

We first bound the local drift for variables $\mathbf{x}$ and $\mathbf{y}$ as

\begin{lemma}\label{lemm1}
Suppose $\eta_{c}^{\mathbf{x}}, \eta_{c}^{\mathbf{y}} \leq \frac{1}{8 K L}$ we have
$$
\mathcal{E}_{t}^{\mathbf{x}} \leq 3 K \Xi_{t}^{\mathbf{x}}+12 K^{2}\left(\eta_{c}^{\mathbf{x}}\right)^{2} L^{2} \mathcal{E}_{t}^{\mathbf{y}}+12 K^{3}\left(\eta_{c}^{\mathbf{x}}\right)^{2} L^{2} \gamma_{t}^{\mathbf{x}}+12 K^{3}\left(\eta_{c}^{\mathbf{x}}\right)^{2} L^{2} \varepsilon_{t}+12 K^{3}\left(\eta_{c}^{\mathbf{x}}\right)^{2} \mathbb{E}\left\|\nabla \Phi(\overline{\mathbf{x}}^{(t)})\right\|^{2}+3 K^{2}\left(\eta_{c}^{\mathbf{x}}\right)^{2} \sigma^{2}
$$
and
$$
\mathcal{E}_{t}^{\mathrm{y}} \leq 3 K \Xi_{t}^{\mathrm{y}}+12 K^{2}\left(\eta_{c}^{\mathrm{y}}\right)^{2} L^{2} \mathcal{E}_{t}^{\mathbf{x}}+12 K^{3}\left(\eta_{c}^{\mathrm{y}}\right)^{2} L^{2} \gamma_{t}^{\mathbf{y}}+6 K^{3}\left(\eta_{c}^{\mathrm{y}}\right)^{2} L^{2} \varepsilon_{t}+3 K^{2}\left(\eta_{c}^{\mathrm{y}}\right)^{2} \sigma^{2}
$$
\end{lemma}

Set $\eta^{\mathbf{x}} \equiv \eta_{s}^{\mathrm{x}} \eta_{c}^{\mathrm{x}}$ and $\eta^{\mathbf{y}} \equiv \eta_{s}^{\mathrm{y}} \eta_{c}^{\mathrm{y}}$.
We now bound the client variance for variables $\mathbf{x}$ and $\mathbf{y}$, as follows

\begin{lemma}\label{lemm2}
We have
$$
\begin{aligned}
& \Xi_{t+1}^{\mathrm{x}} \leq\left(1-\frac{p}{2}\right) \Xi_{t}^{\mathbf{x}}+\frac{6 K\left(\eta^{\mathbf{x}}\right)^{2} L^{2}}{p}\left(\mathcal{E}_{t}^{\mathrm{x}}+\mathcal{E}_{t}^{\mathbf{y}}\right)+\frac{6 K^{2}\left(\eta^{\mathbf{x}}\right)^{2} L^{2}}{p} \gamma_{t}^{\mathbf{x}}+K\left(\eta^{\mathbf{x}}\right)^{2} \sigma^{2} \\
& \Xi_{t+1}^{\mathrm{y}} \leq\left(1-\frac{p}{2}\right) \Xi_{t}^{\mathrm{y}}+\frac{6 K\left(\eta^{\mathbf{y}}\right)^{2} L^{2}}{p}\left(\mathcal{E}_{t}^{\mathrm{x}}+\mathcal{E}_{t}^{\mathrm{y}}\right)+\frac{6 K^{2}\left(\eta^{\mathbf{y}}\right)^{2} L^{2}}{p} \gamma_{t}^{\mathrm{y}}+K\left(\eta^{\mathbf{y}}\right)^{2} \sigma^{2}
\end{aligned}
$$
\end{lemma}

In the upcoming we bound the quality of correction for variables $ \mathbf{x} $ and $ \mathbf{y} $

\begin{lemma}\label{lemm3}
Suppose $\eta^{\mathbf{x}}, \eta^{\mathbf{y}} \leq \frac{\sqrt{p}}{2 \sqrt{6} K L}$ we have
\begin{equation}\label{eq_lemm3}
\begin{aligned}
& \gamma_{t+1}^{\mathbf{x}} \leq\left(1-\frac{p}{2}\right) \gamma_{t}^{\mathbf{x}}+\frac{30}{p K}\left(\mathcal{E}_{t}^{\mathbf{x}}+\mathcal{E}_{t}^{\mathbf{y}}\right)+\frac{12 K^{2} L^{2}}{p}\left(2\left(\eta^{\mathbf{x}}\right)^{2}+\left(\eta^{\mathbf{y}}\right)^{2}\right) \varepsilon_{t}+\frac{24 K^{2}\left(\eta^{\mathbf{x}}\right)^{2}}{p} \mathbb{E}\left\|\nabla \Phi(\overline{\mathbf{x}}^{(t)})\right\|^{2}+\frac{2 \sigma^{2}}{K L^{2}} \\
& \gamma_{t+1}^{\mathbf{y}} \leq\left(1-\frac{p}{2}\right) \gamma_{t}^{\mathbf{y}}+\frac{30}{p K}\left(\mathcal{E}_{t}^{\mathbf{x}}+\mathcal{E}_{t}^{\mathbf{y}}\right)+\frac{12 K^{2} L^{2}}{p}\left(2\left(\eta^{\mathbf{x}}\right)^{2}+\left(\eta^{\mathbf{y}}\right)^{2}\right) \varepsilon_{t}+\frac{24 K^{2}\left(\eta^{\mathbf{x}}\right)^{2}}{p} \mathbb{E}\left\|\nabla \Phi(\overline{\mathbf{x}}^{(t)})\right\|^{2}+\frac{2 \sigma^{2}}{K L^{2}}
\end{aligned}
\end{equation}
\end{lemma}

In the following we bound on the consensus distance for variable $\mathbf{y}$

\begin{lemma}\label{lemm4}
Suppose $\eta^{\mathbf{x}} \leq \frac{\eta^{\mathbf{y}}}{4 \sqrt{6} \kappa^{2}}$ and $\eta^{\mathbf{y}} \leq \frac{1}{K L}$ we have
$$
\varepsilon_{t+1} \leq\left(1-\frac{K \eta^{\mathbf{y}} L}{6 \kappa}\right) \varepsilon_{t}+12 \eta^{\mathbf{y}} L \kappa\left(\mathcal{E}_{t}^{\mathbf{x}}+\mathcal{E}_{t}^{\mathbf{y}}\right)+\frac{16 \kappa^{3} K\left(\eta^{\mathbf{x}}\right)^{2}}{\eta^{\mathbf{y}} L} \mathbb{E}\left\|\nabla \Phi(\overline{\mathbf{x}}^{(t)})\right\|^{2}+\frac{8 \eta^{\mathbf{y}} \kappa}{n L} \sigma^{2}
$$
\end{lemma}

We further have the following bound on the increment of $\mathbb{E} \Phi(\overline{\mathbf{x}}^{(t)})$

\begin{lemma}\label{lemm5}
Suppose $\eta^{\mathbf{x}} \leq \frac{1}{16 K L \kappa}$ we have the following
$$
\mathbb{E}\left[\Phi(\overline{\mathbf{x}}^{(t+1)})-\Phi(\overline{\mathbf{x}}^{(t)})\right] \leq-\frac{\eta^{\mathbf{x}} K}{4} \mathbb{E}\left\|\nabla \Phi(\overline{\mathbf{x}}^{(t)})\right\|^{2}+2 \eta^{\mathbf{x}} L^{2}\left(\mathcal{E}_{t}^{\mathbf{x}}+\mathcal{E}_{t}^{\mathbf{y}}\right)+2 L^{2} \eta^{\mathbf{x}} K \varepsilon_{t}+\frac{K\left(\eta^{\mathbf{x}}\right)^{2} L \kappa}{n} \sigma^{2}
$$
\end{lemma}

Let $v>1$ be a global constant to be determined in our upcoming Lyapunov analysis.
In light of Lemma~\ref{lemm5} we set the Lyapunov function as

\begin{equation}\label{eq_lemm6m}
\mathscr{H}_{t}=\mathbb{E}\left[\Phi(\overline{\mathbf{x}}^{(t)})-\Phi\left(\mathbf{x}^{*}\right)\right]+B^{\mathbf{x}} \eta_{c}^{\mathrm{y}} L \Xi_{t}^{\mathbf{x}}+B^{\mathbf{y}} \eta_{c}^{\mathrm{y}} L \Xi_{t}^{\mathrm{y}}+A^{\mathrm{x}} K^{2} L^{3}\left(\eta_{c}^{\mathrm{y}}\right)^{3} \gamma_{t}^{\mathbf{x}}+A^{\mathrm{y}} K^{2} L^{3}\left(\eta_{c}^{\mathrm{y}}\right)^{3} \gamma_{t}^{\mathrm{y}}+C \frac{1}{K \kappa p} \varepsilon_{t}
\end{equation}
Then we have the following recursion for $\mathscr{H}_{t}$ :

\begin{lemma}\label{lemm6}
Suppose $\eta_{c}^{\mathrm{x}}=\frac{\eta_{c}^{\mathrm{y}}}{\kappa^{2}}$ with $\eta_{c}^{\mathrm{y}}=\frac{p}{300 v \cdot \kappa K L}, \eta_{s}^{\mathrm{x}}=\eta_{s}^{\mathrm{y}}=v \cdot p$, then one can choose global constant $v>1$ such that $A^{\mathrm{x}}=A^{\mathrm{y}}=\frac{108 v^{3}+54 v}{p}, B^{\mathrm{x}}=B^{\mathrm{y}}=\frac{6 v}{p}$ and $C=\frac{1}{24}$ so as to ensure $C_{4}^{\mathbf{x}}=\Theta(1)$ and $C_{4}^{\mathrm{y}}=$ $O(1)$.
We have
\begin{equation}\label{eq_lemm6}
\mathscr{H}_{t+1}-\mathscr{H}_{t} \leq-C_{4}^{\mathrm{x}} K \eta^{\mathbf{x}} \mathbb{E}\left\|\nabla \Phi(\overline{\mathbf{x}}^{(t)})\right\|^{2}+\frac{C_{4}^{y}}{p}(K L)^{2}\left(\eta_{c}^{\mathrm{y}}\right)^{3} \sigma^{2}+\frac{K\left(\eta^{\mathbf{x}}\right)^{2} L \kappa}{n} \sigma^{2}+C \frac{8 \eta^{\mathbf{y}}}{n L K p} \sigma^{2}
\end{equation}
\end{lemma}

With all preliminary lemmas at hand we are ready for the final proof of our main theorem.

\paragraph{Proof of Theorem~\ref{theo1}.}
Taking telescoping sum on both sides of~\eqref{eq_lemm6} in Lemma~\ref{lemm6} for $t=0,1,\ldots, T-1$ gives
$$
\begin{aligned}
& \frac{1}{T} \sum_{t=0}^{T-1}\left(\mathscr{H}_{t+1}-\mathscr{H}_{t}\right)=\frac{1}{T}\left(\mathscr{H}_{T}-\mathscr{H}_{0}\right) \\
& \leq-C_{4}^{\mathbf{x}} K \eta^{\mathbf{x}} \frac{1}{T} \sum_{t=0}^{T-1} \mathbb{E}\left\|\nabla \Phi(\overline{\mathbf{x}}^{(t)})\right\|^{2}+\frac{C_{4}^{\mathbf{y}}}{p}(K L)^{2}\left(\eta_{c}^{\mathbf{x}}\right)^{3} \sigma^{2}+\frac{K\left(\eta^{\mathbf{x}}\right)^{2} L \kappa}{n} \sigma^{2}+C \frac{8 \eta^{\mathbf{y}}}{n L K p} \sigma^{2}
\end{aligned}
$$
yielding
$$
\begin{aligned}
& \frac{1}{T} \sum_{t=0}^{T-1} \mathbb{E}\left\|\nabla \Phi(\overline{\mathbf{x}}^{(t)})\right\|^{2} \leq \frac{\mathscr{H}_{0}-\mathscr{H}_{T}}{T C_{4}^{\mathbf{x}}} \frac{1}{K \eta^{\mathbf{x}}}+\frac{\eta^{\mathbf{x}} L \kappa}{n C_{4}^{\mathbf{x}}} \sigma^{2}+\frac{\frac{C_{4}^{y}}{p} K L^{2}\left(\eta_{c}^{\mathrm{x}}\right)^{3}}{C_{4}^{\times} \eta^{\mathbf{x}}} \sigma^{2}+C \frac{8 \eta^{\mathbf{y}}}{n L C_{4}^{\mathrm{x}} K^{2} p \eta^{\mathbf{x}}} \sigma^{2} \\
& \leq \frac{\mathscr{H}_{0}}{T C_{4}^{\mathbf{x}}} \frac{1}{K \eta^{\mathbf{x}}}+\frac{\eta^{\mathbf{x}} L \kappa}{n C_{4}^{\mathbf{x}}} \sigma^{2}+\frac{C_{4}^{\mathrm{y}} K L^{2}\left(\eta^{\mathbf{x}}\right)^{2}}{C_{4}^{\mathbf{x}} v^{3} p^{4}} \sigma^{2}+\frac{8 C \eta^{\mathbf{y}}}{n L C_{4}^{\mathbf{x}} K^{2} p \eta^{\mathbf{x}}} \sigma^{2}
\end{aligned}
$$
Given desired accuracy $\varepsilon>0$ since we want the expected squared gradient norm of the randomized output $\frac{1}{T} \sum_{t=0}^{T-1} \mathbb{E}\left\|\nabla \Phi(\overline{\mathbf{x}}^{(t)})\right\|^{2} \leq 4 \varepsilon^{2}$ we need to calibrate given the choice of $C_{4}^{\mathbf{x}}, C_{4}^{\mathbf{y}}, \eta^{\mathbf{x}}$ and $\eta^{\mathbf{y}}$.
We make $K \gtrsim \frac{\kappa}{\sqrt{n p}} \frac{\sigma}{\varepsilon}$ so the last term is bounded by $\varepsilon^{2}$.
Using stepsize tuning lemma exemplified by~\cite[Lemma 17]{koloskova2020unified} guarantees the existence of constant stepsize such that the average of accumulation of gradient is upper bounded by
$$
\frac{1}{T} \sum_{t=0}^{T-1} \mathbb{E}\left\|\nabla \Phi(\overline{\mathbf{x}}^{(t)})\right\|^{2} \leq \mathcal{O}\left(\sqrt{\frac{\sigma^{2} L \mathscr{H}_{0}}{n K T}}+\left(\frac{\sigma L \mathscr{H}_{0}}{p^{2} \sqrt{K} T}\right)^{\frac{2}{3}}+\frac{\kappa^{3} L \mathscr{H}_{0}}{p^{2} T}\right)
$$
so the complexity of communication $T$ given local steps $K$ is $\mathcal{O}\left(\frac{\sigma^{2}}{n K} \frac{1}{\varepsilon^{4}}+\frac{\sigma}{p^{2} \sqrt{K}} \frac{1}{\varepsilon^{3}}+\frac{K^{3}}{p^{2}} \frac{1}{\varepsilon^{2}}\right) \cdot L \mathscr{H} \mathscr{0}_{0}$.
Here since the initialization is shared across clients we have $\mathbf{x}^{(0)}=\mathbf{x}_{i}^{(0)}$ for all $i \in[n]$ and the way correction terms $\mathbf{c}_{i}^{\mathbf{x},(0)}, \mathbf{c}_{i}^{\mathbf{y},(0)}$ are defined ensures $\mathscr{H}_{0}=O\left(f\left(\mathbf{x}^{(0)}\right)-f\left(\mathbf{x}^{\star}\right)+\frac{\varepsilon_{0}}{K \kappa p}\right)$ and $\varepsilon_{0}=O\left(\frac{q}{\mu^{2}}\right)$.

\section{Conclusion}\label{sec4}
In conclusion, this work introduces \AlgoName, a pioneering decentralized minimax optimization algorithm for federated learning environments.
By incorporating gradient tracking and local updates, \AlgoName achieves state-of-the-art theoretical communication efficiency, showcasing its superiority over existing methods in terms of convergence rates and scalability.
We hope our work addresses the critical challenges of data heterogeneity, communication efficiency and data heterogeneity robustness in federated learning setups.


\bibliographystyle{alpha}
\bibliography{ref}

\newcommand{\Lyap}{\mathscr{H}}
\newcommand{\Devi}{\varepsilon}
\appendix
\section{Deferred Auxiliary Proofs}\label{sec_aux_proof}
\subsection{Proof of Lemma~\ref{lemm1}}\label{sec_proof,lemm1}
\begin{lemma}\label{lemm7}
Using Assumption~\ref{assu2} and Young's Inequality we have
$$
\mathbb{E}\left\|\nabla_{\mathbf{x}} f(\overline{\mathbf{x}}^{(t)}, \overline{\mathbf{y}}^{(t)})\right\|^{2} \leq 2 L^{2} \varepsilon_{t}+2 \mathbb{E}\left\|\nabla \Phi(\overline{\mathbf{x}}^{(t)})\right\|^{2} \quad \mathbb{E}\left\|\nabla_{\mathbf{y}} f(\overline{\mathbf{x}}^{(t)}, \overline{\mathbf{y}}^{(t)})\right\|^{2} \leq L^{2} \varepsilon_{t}
$$
\end{lemma}

\begin{proof}[Proof of Lemma~\ref{lemm7}]
We have
$$
\begin{aligned}
\mathbb{E}\left\|\nabla_{\mathbf{x}} f(\overline{\mathbf{x}}^{(t)}, \overline{\mathbf{y}}^{(t)})\right\|^{2} & =\mathbb{E}\left\|\nabla_{\mathbf{x}} f(\overline{\mathbf{x}}^{(t)}, \overline{\mathbf{y}}^{(t)})-\nabla_{\mathbf{x}} f\left(\overline{\mathbf{x}}^{(t)}, \hat{\mathbf{y}}^{(t)}\right)+\nabla_{\mathbf{x}} f\left(\overline{\mathbf{x}}^{(t)}, \hat{\mathbf{y}}^{(t)}\right)\right\|^{2} \\
& \leq 2 L^{2} \mathbb{E}\left\|\overline{\mathbf{y}}^{(t)}-\hat{\mathbf{y}}^{(t)}\right\|^{2}+2 \mathbb{E}\left\|\nabla \Phi(\overline{\mathbf{x}}^{(t)})\right\|^{2}=2 L^{2} \varepsilon_{t}+2 \mathbb{E}\left\|\nabla \Phi(\overline{\mathbf{x}}^{(t)})\right\|^{2}
\end{aligned}
$$
Moreover
\begin{equation}\label{eq_proof_lemm7}
\mathbb{E}\left\|\nabla_{\mathbf{y}} f(\overline{\mathbf{x}}^{(t)}, \overline{\mathbf{y}}^{(t)})\right\|^{2}=\mathbb{E}\left\|\nabla_{\mathbf{y}} f(\overline{\mathbf{x}}^{(t)}, \overline{\mathbf{y}}^{(t)})-\nabla_{\mathbf{y}} f\left(\overline{\mathbf{x}}^{(t)}, \hat{\mathbf{y}}^{(t)}\right)\right\|^{2} \leq L^{2} \varepsilon_{t}
\end{equation}
The equality in~\eqref{eq_proof_lemm7} holds due to the fact that $\nabla_{\mathbf{y}} f(\bar{\mathbf{x}}^{(t)}, \hat{\mathbf{y}}^{(t)}) = 0$.
\end{proof}

\begin{proof}[Proof of Lemma~\ref{lemm1}]
For $K=1$ the inequalities obviously hold since $\mathcal{E}_{t}^{\mathbf{x}}=\Xi_{t}^{\mathbf{x}}=\frac{1}{n} \mathbb{E}\left\|\mathbf{X}^{(t)}-\overline{\mathbf{X}}^{(t)}\right\|_{F}^{2}$ and $\mathcal{E}_{t}^{\mathbf{y}}=\Xi_{t}^{\mathbf{y}}=\frac{1}{n} \mathbb{E}\left\|\mathbf{Y}^{(t)}-\overline{\mathbf{Y}}^{(t)}\right\|_{F}^{2}$ and other terms on the RHSs are positive.
For $K \geq 2$ we have
\begin{small}
$$\hspace{-.1in}
\begin{aligned}
& n e_{k, t}^{\mathbf{x}} \equiv \mathbb{E}\left\|\mathbf{X}^{(t)+k}-\overline{\mathbf{X}}^{(t)}\right\|_{F}^{2} \\
& =\mathbb{E}\left\|\mathbf{X}^{(t)+k-1}-\eta_{c}^{\mathbf{x}}\left(\nabla_{\mathbf{x}} F(\mathbf{X}^{(t)+k-1}, \mathbf{Y}^{(t)+k-1} ; \xi^{(t)+k-1})+\mathbf{C}^{x,(t)}\right)-\overline{\mathbf{X}}^{(t)}\right\|_{F}^{2} \\
& \leq\left(1+\frac{1}{K-1}\right) \mathbb{E}\left\|\mathbf{X}^{(t)+k-1}-\overline{\mathbf{X}}^{(t)}\right\|^{2}+n\left(\eta_{c}^{\mathbf{x}}\right)^{2} \sigma^{2} \\
& \quad+K\left(\eta_{c}^{\mathbf{x}}\right)^{2} \mathbb{E} \| \nabla_{\mathbf{x}} f(\mathbf{X}^{(t)+k-1}, \mathbf{Y}^{(t)+k-1})-\nabla_{\mathbf{x}} f(\overline{\mathbf{x}}^{(t)}, \overline{\mathbf{y}}^{(t)})+\mathbf{C}^{x,(t)} \\
& \hspace{1.1in}+\nabla_{\mathbf{x}} f(\overline{\mathbf{x}}^{(t)}, \overline{\mathbf{y}}^{(t)})(\mathbf{I}-\mathbf{J})+\nabla_{\mathbf{x}} f(\overline{\mathbf{x}}^{(t)}, \overline{\mathbf{y}}^{(t)}) \mathbf{J} \|_{F}^{2} \\
& \leq \underbrace{\left(1+\frac{1}{K-1}+4 K\left(\eta_{c}^{\mathbf{x}}\right)^{2} L^{2}\right)}_{\equiv q} \mathbb{E}\left\|\mathbf{X}^{(t)+k-1}-\overline{\mathbf{X}}^{(t)}\right\|_{F}^{2}+4 K\left(\eta_{c}^{\mathbf{x}}\right)^{2} L^{2} \mathbb{E}\left\|\mathbf{Y}^{(t)+k-1}-\overline{\mathbf{Y}}^{(t)}\right\|_{F}^{2} \\
& \quad+4 K\left(\eta_{c}^{\mathbf{x}}\right)^{2} L^{2} n \gamma_{t}^{\mathbf{x}}+2 K\left(\eta_{c}^{\mathbf{x}}\right)^{2} n \mathbb{E}\left\|\nabla_{\mathbf{x}} f(\overline{\mathbf{x}}^{(t)}, \overline{\mathbf{y}}^{(t)})\right\|^{2}+n\left(\eta_{c}^{\mathbf{x}}\right)^{2} \sigma^{2} \\
& \leq q^{k} \mathbb{E}\left\|\mathbf{X}^{(t)}-\overline{\mathbf{X}}^{(t)}\right\|_{F}^{2} \\
& \quad+\sum_{r=0}^{k-1} q^{r}\left(4 K\left(\eta_{c}^{\mathbf{x}}\right)^{2} L^{2} \mathbb{E}\left\|\mathbf{Y}^{(t)+k-1}-\overline{\mathbf{Y}}^{(t)}\right\|_{F}^{2}+4 K\left(\eta_{c}^{\mathbf{x}}\right)^{2} L^{2} n \gamma_{t}^{\mathbf{x}}+2 K\left(\eta_{c}^{\mathbf{x}}\right)^{2} n \mathbb{E}\left\|\nabla_{\mathbf{x}} f(\overline{\mathbf{x}}^{(t)}, \overline{\mathbf{y}}^{(t)})\right\|^{2}+n\left(\eta_{c}^{\mathbf{x}}\right)^{2} \sigma^{2}\right)
\end{aligned}
$$
\end{small}%
If the condition $\eta_{c}^{\mathrm{x}} \leq \frac{1}{8 K L}$ holds, then it follows that $4 K\left(\eta_{c}^{\mathrm{x}} L\right)^{2} \leq \frac{1}{16 K}<\frac{1}{16(K-1)}$.
Given $q>1$, it can be established that $q^{k} \leq q^{K} \leq\left(1+\frac{1}{K-1}+\frac{1}{16(K-1)}\right)^{K} \leq e^{1+\frac{1}{16}} \leq 3$, and $\sum_{r}^{k-1} q^{r} \leq K q^{K} \leq 3 K$.
We can obtain a bound on client drift for variable $\mathbf{x}$
\begin{small}
$$
\mathcal{E}_{t}^{\mathbf{x}}=\sum_{k=0}^{K-1} e_{k, t}^{\mathbf{x}} \leq 3 K \Xi_{t}^{\mathbf{x}}+12 K^{2}\left(\eta_{c}^{\mathbf{x}}\right)^{2} L^{2} \mathcal{E}_{t}^{\mathbf{y}}+12 K^{3}\left(\eta_{c}^{\mathbf{x}}\right)^{2} L^{2} \gamma_{t}^{\mathbf{x}}+6 K^{3}\left(\eta_{c}^{\mathbf{x}}\right)^{2} \mathbb{E}\left\|\nabla_{\mathbf{x}} f(\overline{\mathbf{x}}^{(t)}, \overline{\mathbf{y}}^{(t)})\right\|^{2}+3 K^{2}\left(\eta_{c}^{\mathbf{x}}\right)^{2} \sigma^{2}
$$
\end{small}%
A bound on client drift for variable $\mathbf{y}$ can be analogously obtained
\begin{small}
$$
\mathcal{E}_{t}^{\mathbf{y}}=\sum_{k=0}^{K-1} e_{k, t}^{\mathbf{y}} \leq 3 K \Xi_{t}^{\mathbf{y}}+12 K^{2}\left(\eta_{c}^{\mathbf{y}}\right)^{2} L^{2} \mathcal{E}_{t}^{\mathbf{x}}+12 K^{3}\left(\eta_{c}^{\mathbf{y}}\right)^{2} L^{2} \gamma_{t}^{\mathbf{y}}+6 K^{3}\left(\eta_{c}^{\mathbf{y}}\right)^{2} \mathbb{E}\left\|\nabla_{\mathbf{y}} f(\overline{\mathbf{x}}^{(t)}, \overline{\mathbf{y}}^{(t)})\right\|^{2}+3 K^{2}\left(\eta_{c}^{\mathbf{y}}\right)^{2} \sigma^{2}
$$
\end{small}%
Lemma~\ref{lemm7} along with the above two displays completes the proof.
\end{proof}

\subsection{Proof of Lemma~\ref{lemm2}}\label{sec_proof,lemm2}
\begin{proof}[Proof of Lemma~\ref{lemm2}]
Using the update rule from Algorithm~\ref{algo1} we can bound the client variance for variable $\mathbf{x}$
$$
\begin{aligned}
& n \Xi_{t+1}^{\mathbf{x}}=\mathbb{E}\left\|\mathbf{X}^{(t+1)}-\overline{\mathbf{X}}^{(t+1)}\right\|^{2} \\
= & \mathbb{E}\left\|\left(\mathbf{X}^{(t)}-\eta^{\mathbf{x}} \sum_{k=0}^{K-1}\left(\nabla_{\mathbf{x}} F\left(\mathbf{X}^{(t)+k}, \mathbf{Y}^{(t)+k} ; \xi^{(t)+k}\right)+\mathbf{C}^{x,(t)}\right)\right)(\mathbf{W}-\mathbf{J})\right\|_{F}^{2} \\
\leq & (1-p) \mathbb{E}\left\|\left(\mathbf{X}^{(t)}-\eta^{\times} \sum_{k=0}^{K-1}\left(\nabla_{\mathbf{x}} f\left(\mathbf{X}^{(t)+k}, \mathbf{Y}^{(t)+k}\right)+\mathbf{C}^{x,(t)}\right)\right)(\mathbf{I}-\mathbf{J})\right\|_{F}^{2}+n K\left(\eta^{\times}\right)^{2} \sigma^{2} \\
\leq & n K\left(\eta^{\times}\right)^{2} \sigma^{2}+(1+\alpha)(1-p) \mathbb{E}\left\|\mathbf{X}^{(t)}(\mathbf{I}-\mathbf{J})\right\|_{F}^{2} \\
& +\left(1+\frac{1}{\alpha}\right)\left(\eta^{\mathbf{x}}\right)^{2} \mathbb{E}\left\|\left[\sum_{k=0}^{K-1} \nabla_{\mathbf{x}} f\left(\mathbf{X}^{(t)+k}, \mathbf{Y}^{(t)+k}\right)-K \nabla_{\mathbf{x}} f(\overline{\mathbf{x}}^{(t)}, \overline{\mathbf{y}}^{(t)})+K \nabla_{\mathbf{x}} f(\overline{\mathbf{x}}^{(t)}, \overline{\mathbf{y}}^{(t)})\right](\mathbf{I}-\mathbf{J})+K \mathbf{C}^{x,(t)}\right\|_{F}^{2} \\
\leq & n K\left(\eta^{\times}\right)^{2} \sigma^{2}+\left(1-\frac{p}{2}\right) \mathbb{E}\left\|\mathbf{X}^{(t)}-\overline{\mathbf{X}}^{(t)}\right\|_{F}^{2} \\
& +\frac{6}{p}\left(K\left(\eta^{\mathbf{x}}\right)^{2} L^{2}\|\mathbf{I}-\mathbf{J}\|^{2}\left(\sum_{k=0}^{K-1}\left\|\mathbf{X}^{(t)+k}-\overline{\mathbf{X}}^{(t)}\right\|_{F}^{2}+\sum_{k=0}^{K-1} \mathbb{E}\left\|\mathbf{Y}^{(t)+k}-\overline{\mathbf{Y}}^{(t)}\right\|_{F}^{2}\right)\right.
\\
& \left.++K^{2}\left(\eta^{\mathbf{x}}\right)^{2} \mathbb{E}\left\|\nabla_{\mathbf{x}} f(\overline{\mathbf{x}}^{(t)}, \overline{\mathbf{y}}^{(t)})(\mathbf{I}-\mathbf{J})+\mathbf{C}^{x,(t)}\right\|_{F}^{2}\right) \\
\leq & \left(1-\frac{p}{2}\right) n \Xi_{t}^{\mathbf{x}}+\frac{6 K\left(\eta^{\mathbf{x}}\right)^{2} L^{2}}{p} n\left(\mathcal{E}_{t}^{\mathbf{x}}+\mathcal{E}_{t}^{y}\right)+\frac{6 K^{2}\left(\eta^{\mathbf{x}}\right)^{2} L^{2}}{p} n \gamma_{t}^{\mathbf{x}}+n K\left(\eta^{\times}\right)^{2} \sigma^{2}
\end{aligned}
$$
where we used Assumption~\ref{assu4} and $\alpha=\frac{p}{2}$, $p \leq 1$.
We also derive an upper bound on client variance for variable $\mathbf{y}$, completing the proof.
\end{proof}

\subsection{Proof of Lemma~\ref{lemm3}}\label{sec_proof,lemm3}
\begin{lemma}\label{lemm8}
If we initialize $\mathbf{C}^{x,(0)}$ and $\mathbf{C}^{y,(0)}$ as below
\begin{equation}\label{eq_lemm8}
\begin{aligned}
& \mathbf{c}_{i}^{\mathbf{x}(0)}=-\nabla_{\mathbf{x}} F_{i}(\mathbf{x}^{(0)}, \mathbf{y}^{(0)} ; \xi_{i})+\frac{1}{n} \sum_{j=1}^{n} \nabla_{\mathbf{x}} F_{j}(\mathbf{x}^{(0)}, \mathbf{y}^{(0)} ; \xi_{j})
\\
& \mathbf{c}_{i}^{\mathbf{y},(0)}=-\nabla_{\mathbf{y}} F_{i}(\mathbf{x}^{(0)}, \mathbf{y}^{(0)} ; \xi_{i})+\frac{1}{n} \sum_{j=1}^{n} \nabla_{\mathbf{y}} F_{j}(\mathbf{x}^{(0)}, \mathbf{y}^{(0)} ; \xi_{j})
\end{aligned}
\end{equation}
then the averaged correction for variables $\mathbf{x}$ and $\mathbf{y}$ in any communication round equals to zero.
\end{lemma}

\begin{proof}[Proof of Lemma~\ref{lemm8}]
According to Algorithm~\ref{algo1} we have
$$
\mathbf{C}^{x,(t+1)} \mathbf{J}=\mathbf{C}^{x,(t)} \mathbf{J}+\frac{1}{K \eta_{c}^{x}}\left(\mathbf{X}^{(t)}-\mathbf{X}^{(t)+K}\right)(\mathbf{W}-\mathbf{I}) \mathbf{J}=\mathbf{C}^{x,(t)} \mathbf{J}
$$
Using the initialization assumption as in~\eqref{eq_lemm8} we have $\mathbf{C}^{x,(t)} \mathbf{J}=\mathbf{C}^{x,(0)} \mathbf{J}=\mathbf{0}$.
We analogously have $\mathbf{C}^{y,(t)} \mathbf{J}=\mathbf{C}^{y,(0)} \mathbf{J}=\mathbf{0}$.
\end{proof}

Let $\Delta_{t+1}^{\mathbf{x}} \equiv \mathbb{E}\left\|\overline{\mathbf{x}}^{(t+1)}-\overline{\mathbf{x}}^{(t)}\right\|^{2}, \Delta_{t+1}^{\mathbf{y}} \equiv \mathbb{E}\left\|\overline{\mathbf{y}}^{(t+1)}-\overline{\mathbf{y}}^{(t)}\right\|^{2}$.
We have

\begin{lemma}\label{lemm9}
The sum of averaged progress between communications for variables $\mathbf{x}$ and $\mathbf{y}$ can be bounded by
$$
\begin{aligned}
\Delta_{t+1}^{\mathrm{x}}+\Delta_{t+1}^{\mathrm{y}} \leq & 2 K L^{2}\left(\left(\eta^{\mathbf{x}}\right)^{2}+\left(\eta^{\mathbf{y}}\right)^{2}\right)\left(\mathcal{E}_{t}^{\mathbf{x}}+\mathcal{E}_{t}^{\mathbf{y}}\right)+2 K^{2} L^{2}\left(2\left(\eta^{\mathbf{x}}\right)^{2}+\left(\eta^{\mathbf{y}}\right)^{2}\right) \varepsilon_{t} \\
& +4 K^{2}\left(\eta^{\mathbf{x}}\right)^{2} \mathbb{E}\left\|\nabla \Phi\left(\bar{x}^{(t)}\right)\right\|^{2}+\frac{K \sigma^{2}}{n}\left(\left(\eta^{\mathbf{x}}\right)^{2}+\left(\eta^{\mathbf{y}}\right)^{2}\right)
\end{aligned}
$$
\end{lemma}

\begin{proof}[Proof of Lemma~\ref{lemm9}]
First, we derive an upper bound on the averaged progress for variable $\mathbf{x}$ as follows
\begin{equation}\label{eq_proof_lemm9}
\begin{aligned}
& \Delta_{t+1}^{\mathbf{x}}=\mathbb{E}\left\|\overline{\mathbf{x}}^{(t+1)}-\overline{\mathbf{x}}^{(t)}\right\|^{2}=\left(\eta^{\mathbf{x}}\right)^{2} \mathbb{E}\left\|\frac{1}{n} \sum_{i, k} \nabla_{\mathbf{x}} F_{i}(\mathbf{x}_{i}^{(t)+k}, \mathbf{y}_{i}^{(t)+k} ; \xi^{(t)+k})+\frac{K}{n} \sum_{i} \mathbf{c}_{i}^{\mathbf{x},(t)}\right\|^{2} \\
& \leq \frac{2 K\left(\eta^{\mathbf{x}}\right)^{2}}{n} \sum_{i, k} \mathbb{E}\left\|\nabla_{\mathbf{x}} f_{i}(\mathbf{x}_{i}^{(t)+k}, \mathbf{y}_{i}^{(t)+k})-\nabla_{\mathbf{x}} f_{i}(\overline{\mathbf{x}}_{i}^{(t)}, \overline{\mathbf{y}}_{i}^{(t)})\right\|^{2}+2 K^{2}\left(\eta^{\mathbf{x}}\right)^{2} \mathbb{E}\left\|\nabla_{\mathbf{x}} f(\overline{\mathbf{x}}_{i}^{(t)}, \overline{\mathbf{y}}_{i}^{(t)})\right\|^{2}+\frac{K \eta_{x}^{2} \sigma^{2}}{n} \\
& \leq \frac{2 K\left(\eta^{\mathbf{x}}\right)^{2} L^{2}}{n} \sum_{i, k}\left(\mathbb{E}\left\|\mathbf{x}_{i}^{(t)+k}-\overline{\mathbf{x}}^{(t)}\right\|^{2}+\mathbb{E}\left\|\mathbf{y}_{i}^{(t)+k}-\overline{\mathbf{y}}^{(t)}\right\|^{2}\right)+2 K^{2}\left(\eta^{\mathbf{x}}\right)^{2} \mathbb{E}\left\|\nabla_{\mathbf{x}} f(\overline{\mathbf{x}}_{i}^{(t)}, \overline{\mathbf{y}}_{i}^{(t)})\right\|^{2}+\frac{K \eta_{x}^{2} \sigma^{2}}{n} \\
& \leq 2 K\left(\eta^{\mathbf{x}}\right)^{2} L^{2}\left(\mathcal{E}_{t}^{\mathbf{x}}+\mathcal{E}_{t}^{\mathbf{y}}\right)+2 K^{2}\left(\eta^{\mathbf{x}}\right)^{2}\left(2 L^{2} \varepsilon_{t}+2 \mathbb{E}\left\|\nabla \Phi(\overline{\mathbf{x}}^{(t)})\right\|^{2}\right)+\frac{K \eta_{x}^{2} \sigma^{2}}{n}
\end{aligned}
\end{equation}
Analogous to the above derivation we have
\begin{equation}\label{eq_proof_lemm9prime}
\begin{aligned}
& \Delta_{t+1}^{\mathbf{y}}=\mathbb{E}\left\|\overline{\mathbf{y}}^{(t+1)}-\overline{\mathbf{y}}^{(t)}\right\|^{2} \leq 2 K^{2}\left(\eta^{\mathbf{y}}\right)^{2} L^{2}\left(\mathcal{E}_{t}^{\mathbf{x}}+\mathcal{E}_{t}^{\mathbf{y}}\right)+2 K^{2}\left(\eta^{\mathbf{y}}\right)^{2} \mathbb{E}\left\|\nabla_{\mathbf{y}} f(\overline{\mathbf{x}}_{i}^{(t)}, \overline{\mathbf{y}}_{i}^{(t)})\right\|^{2}+\frac{K\left(\eta^{\mathbf{y}}\right)^{2} \sigma^{2}}{n} \\
& \leq 2 K\left(\eta^{\mathbf{y}}\right)^{2} L^{2}\left(\mathcal{E}_{t}^{\mathbf{x}}+\mathcal{E}_{t}^{\mathbf{y}}\right)+2 K^{2}\left(\eta^{\mathbf{y}}\right)^{2} L^{2} \varepsilon_{t}+\frac{K\left(\eta^{y}\right)^{2} \sigma^{2}}{n}
\end{aligned}
\end{equation}
where we applied Lemmas~\ref{lemm1} and~\ref{lemm8}. Combining~\eqref{eq_proof_lemm9} and~\eqref{eq_proof_lemm9prime} completes the proof.
\end{proof}

\begin{proof}[Proof of Lemma~\ref{lemm3}]
We have
$$
\begin{aligned}
& n L^{2} \gamma_{t+1}^{\mathbf{x}}-\frac{n \sigma^{2}}{K} \\
& \equiv \mathbb{E}\left\|\mathbf{C}^{x,(t+1)}+\nabla_{\mathbf{x}} f\left(\overline{\mathbf{X}}^{(t+1)}, \overline{\mathbf{Y}}^{(t+1)}\right)(\mathbf{I}-\mathbf{J})\right\|_{F}^{2}-\frac{n \sigma^{2}}{K} \\
& =\mathbb{E}\left\|\mathbf{C}^{x,(t)} \mathbf{W}+\frac{1}{K} \sum_{k=0}^{K-1} \nabla_{\mathbf{x}} F\left(\mathbf{X}^{(t)+k}, \mathbf{Y}^{(t)+k} ; \xi^{(t)+k}\right)(\mathbf{W}-\mathbf{I})+\nabla_{\mathbf{x}} f\left(\overline{\mathbf{X}}^{(t+1)}, \overline{\mathbf{Y}}^{(t+1)}\right)(\mathbf{I}-\mathbf{J})\right\|_{F}^{2}-\frac{n \sigma^{2}}{K} \\
& \leq \mathbb{E} \|\left(\mathbf{C}^{x,(t)}+\nabla_{\mathbf{x}} f(\overline{\mathbf{x}}^{(t)}, \overline{\mathbf{y}}^{(t)})(\mathbf{I}-\mathbf{J})\right) \mathbf{W}+\left(\frac{1}{K} \sum_{k=0}^{K-1} \nabla_{\mathbf{x}} f\left(\mathbf{X}^{(t)+k}, \mathbf{Y}^{(t)+k}\right)-\nabla_{\mathbf{x}} f(\overline{\mathbf{x}}^{(t)}, \overline{\mathbf{y}}^{(t)})\right)(\mathbf{W}-\mathbf{I}) \\
& +\left(\nabla_{\mathbf{x}} f\left(\overline{\mathbf{X}}^{(t+1)}, \overline{\mathbf{Y}}^{(t+1)}\right)-\nabla_{\mathbf{x}} f(\overline{\mathbf{x}}^{(t)}, \overline{\mathbf{y}}^{(t)})\right)(\mathbf{I}-\mathbf{J}) \|_{F}^{2} \\
& \leq(1+\alpha)(1-p) n L^{2} \gamma_{t}^{\mathbf{x}} \quad \\
& +2\left(1+\frac{1}{\alpha}\right)\left[\|\mathbf{W}-\mathbf{I}\|^{2} \frac{L^{2}}{K} \sum_{k=0}^{K-1}\left(\mathbb{E}\left\|\mathbf{X}^{(t)+k}-\overline{\mathbf{X}}^{(t)}\right\|^{2}+\mathbb{E}\left\|\mathbf{Y}^{(t)+k}-\overline{\mathbf{Y}}^{(t)}\right\|^{2}\right)\right.
\\
& \left.+\|\mathbf{I}-\mathbf{J}\|^{2} n L^{2}\left(\mathbb{E}\left\|\overline{\mathbf{x}}^{(t+1)}-\overline{\mathbf{x}}^{(t)}\right\|^{2}+\mathbb{E}\left\|\overline{\mathbf{y}}^{(t+1)}-\overline{\mathbf{y}}^{(t)}\right\|^{2}\right)\right] \\
& \leq\left(1-\frac{p}{2}\right) n L^{2} \gamma_{t}^{\mathrm{x}}+\frac{6}{p}\left(\frac{4 L^{2} n}{K}\left(\mathcal{E}_{t}^{\mathrm{x}}+\mathcal{E}_{t}^{\mathrm{y}}\right)+n L^{2}\left(\Delta_{t+1}^{\mathrm{x}}+\Delta_{t+1}^{\mathrm{y}}\right)\right)
\end{aligned}
$$
where we applied $\alpha=\frac{p}{2}$, $\frac{1}{p} \geq 1$ and Assumption~\ref{assu4}, along with the fact that
$$
\left(\mathbf{C}^{x,(t)}+\nabla_{\mathbf{x}} f(\overline{\mathbf{x}}^{(t)}, \overline{\mathbf{y}}^{(t)})(\mathbf{I}-\mathbf{J})\right) \mathbf{J}=\mathbf{C}^{x,(t)} \mathbf{J}+\nabla_{\mathbf{x}} f(\overline{\mathbf{x}}^{(t)}, \overline{\mathbf{y}}^{(t)})(\mathbf{J}-\mathbf{J})=\mathbf{0}
$$
where in the last equality we used Lemma~\ref{lemm8}.
Using Lemma~\ref{lemm9} to bound $\Delta_{t+1}^{\mathrm{x}}+\Delta_{t+1}^{\mathrm{y}}$ we have
$$
\begin{aligned}
\gamma_{t+1}^{\mathbf{x}} \leq & \left(1-\frac{p}{2}\right) \gamma_{t}^{\mathbf{x}}+\frac{1}{p}\left(\frac{24}{K}+12 K\left(\eta^{\mathbf{x}}\right)^{2} L^{2}+12 K\left(\eta^{\mathbf{y}}\right)^{2} L^{2}\right)\left(\mathcal{E}_{t}^{\mathbf{x}}+\mathcal{E}_{t}^{\mathbf{y}}\right) \\
& +\frac{12 K^{2} L^{2}}{p}\left(2\left(\eta^{\mathbf{x}}\right)^{2}+\left(\eta^{\mathbf{y}}\right)^{2}\right) \varepsilon_{t}+\frac{24 K^{2}\left(\eta^{\mathbf{x}}\right)^{2}}{p} \mathbb{E}\left\|\nabla \Phi(\overline{\mathbf{x}}^{(t)})\right\|^{2}+\frac{6 K \sigma^{2}\left(\left(\eta^{\mathbf{x}}\right)^{2}+\left(\eta^{\mathbf{y}}\right)^{2}\right)}{n p}+\frac{\sigma^{2}}{K L^{2}}
\end{aligned}
$$
Conditions on the step sizes yields~\eqref{eq_lemm3} for $\mathbf{x}$, and analogously for $\mathbf{y}$.
\end{proof}

\subsection{Proof of Lemma~\ref{lemm4}}
Using Proposition~\ref{prop2} we have the following bound
\begin{lemma}\label{lemm10}
Assuming $\eta^{y} \leq \frac{1}{K L}$, then for any $\alpha>0$
$$
\mathbb{E}\left\|\hat{\mathbf{y}}^{(t)}-\overline{\mathbf{y}}^{(t+1)}\right\|^{2} \leq(1+\alpha)\left(1-K \eta^{\mathbf{y}} \mu\right) \varepsilon_{t}+\left(1+\frac{1}{\alpha}\right)\left(\eta^{\mathbf{y}}\right)^{2} L^{2} K\left(\mathcal{E}^{\mathbf{x}}+\mathcal{E}^{\mathbf{y}}\right)+\frac{K \eta_{y}^{2} \sigma^{2}}{n}
$$
\end{lemma}

\begin{proof}[Proof of Lemma~\ref{lemm10}]
If we replace $\mathbf{x}=\overline{\mathbf{x}}^{(t)}$, $\mathbf{y}=\overline{\mathbf{y}}^{(t)}$, and $\mathbf{y}^{\prime}=\hat{\mathbf{y}}^{(t)}$ in Proposition~\ref{prop2}, we have
\begin{equation}\label{eq_lemm10}
\nabla_{\mathbf{y}} f(\overline{\mathbf{x}}^{(t)}, \overline{\mathbf{y}}^{(t)})^{\top}\left(\overline{\mathbf{y}}^{(t)}-\hat{\mathbf{y}}^{(t)}\right)+\frac{1}{2 L}\left\|\nabla_{\mathbf{y}} f(\overline{\mathbf{x}}^{(t)}, \overline{\mathbf{y}}^{(t)})\right\|^{2}+\frac{\mu}{2}\left\|\overline{\mathbf{y}}^{(t)}-\hat{\mathbf{y}}^{(t)}\right\|^{2} \leq 0
\end{equation}
We also have
$$
\begin{aligned}
& \mathbb{E}\left\|\hat{\mathbf{y}}^{(t)}-\overline{\mathbf{y}}^{(t)}-K \eta^{\mathbf{y}} \nabla_{\mathbf{y}} f(\overline{\mathbf{x}}^{(t)}, \overline{\mathbf{y}}^{(t)})\right\|^{2} \\
& =\mathbb{E}\left\|\hat{\mathbf{y}}^{(t)}-\overline{\mathbf{y}}^{(t)}\right\|^{2}-2 K \eta y \mathbb{E}\left\langle\hat{\mathbf{y}}^{(t)}-\overline{\mathbf{y}}^{(t)}, \nabla y f(\overline{\mathbf{x}}^{(t)}, \overline{\mathbf{y}}^{(t)})\right\rangle+K^{2}\left(\eta^{\mathbf{y}}\right)^{2} \mathbb{E}\left\|\nabla_{\mathbf{y}} f(\overline{\mathbf{x}}^{(t)}, \overline{\mathbf{y}}^{(t)})\right\|^{2} \\
& =\mathbb{E}\left\|\hat{\mathbf{y}}^{(t)}-\overline{\mathbf{y}}^{(t)}\right\|^{2}+2 K \eta^{\mathbf{y}}\left(\mathbb{E}\left\langle\overline{\mathbf{y}}^{(t)}-\hat{\mathbf{y}}(t), \nabla_{\mathbf{y}} f(\overline{\mathbf{x}}^{(t)}, \overline{\mathbf{y}}^{(t)})\right\rangle+\frac{K \eta^{\mathbf{y}}}{2} \mathbb{E}\left\|\nabla_{\mathbf{y}} f(\overline{\mathbf{x}}^{(t)}, \overline{\mathbf{y}}^{(t)})\right\|^{2}\right) \\
& \leq \mathbb{E}\left\|\hat{\mathbf{y}}^{(t)}-\overline{\mathbf{y}}^{(t)}\right\|^{2}+2 K \eta^{\mathbf{y}}\left(-\frac{\mu}{2} \mathbb{E}\left\|\hat{\mathbf{y}}^{(t)}-\overline{\mathbf{y}}^{(t)}\right\|^{2}\right)=\left(1-K \eta^{\mathbf{y}} \mu\right) \varepsilon_{t}
\end{aligned}
$$
where we used $\leq \frac{1}{K\Lip} $ and~\eqref{eq_lemm10}.
We have
$$
\begin{aligned}
& \mathbb{E}\left\|\hat{\mathbf{y}}^{(t)}-\overline{\mathbf{y}}^{(t+1)}\right\|^{2}-\frac{K \eta_{y}^{2} \sigma^{2}}{n}=\mathbb{E}\left\|\hat{\mathbf{y}}^{(t)}-\overline{\mathbf{y}}^{(t)}-\frac{\eta^{\mathbf{y}}}{n} \sum_{i, k} \nabla_{\mathbf{y}} F_{i}(\mathbf{x}_{i}^{(t)+k}, \mathbf{y}_{i}^{(t)+k} ; \xi^{(t)+k})\right\|^{2}-\frac{K \eta_{y}^{2} \sigma^{2}}{n} \\
& \leq \mathbb{E}\left\|\hat{\mathbf{y}}^{(t)}-\overline{\mathbf{y}}^{(t)}-K \eta^{\mathbf{y}} \nabla_{\mathbf{y}} f(\overline{\mathbf{x}}^{(t)}, \overline{\mathbf{y}}^{(t)})-\frac{\eta^{\mathbf{y}}}{n} \sum_{i, k} \nabla_{\mathbf{y}} f_{i}(\mathbf{x}_{i}^{(t)+k}, \mathbf{y}_{i}^{(t)+k})+\frac{\eta^{y}}{n} \sum_{i, k} \nabla_{\mathbf{y}} f_{i}(\overline{\mathbf{x}}^{(t)}, \overline{\mathbf{y}}^{(t)})\right\|^{2} \\
& \leq(1+\alpha) \mathbb{E}\left\|\hat{\mathbf{y}}^{(t)}-\overline{\mathbf{y}}^{(t)}-K \eta^{\mathbf{y}} \nabla_{\mathbf{y}} f(\overline{\mathbf{x}}^{(t)}, \overline{\mathbf{y}}^{(t)})\right\|^{2} \\
&+\left(1+\frac{1}{\alpha}\right) \frac{\left(\eta^{\mathbf{y}}\right)^{2} K}{n} \sum_{i, k} \mathbb{E}\left\|\nabla_{\mathbf{y}} f_{i}\left(\mathbf{x}_{i}^{(t)+k}, y_{i}^{(t)+k}\right)-\nabla_{\mathbf{y}} f_{i}(\overline{\mathbf{x}}^{(t)}, \overline{\mathbf{y}}^{(t)})\right\|^{2} \\
& \leq(1+\alpha)\left(1-K \eta^{\mathbf{y}} \mu\right) \varepsilon_{t}+\left(1+\frac{1}{\alpha}\right)\left(\eta^{\mathbf{y}}\right)^{2} L^{2} K\left(\mathcal{E}^{\mathbf{x}}+\mathcal{E}^{\mathbf{y}}\right)
\end{aligned}
$$
where we used Lemma~\ref{lemm8} i.e., $\frac{1}{n} \sum_{i} \mathbf{c}_{i}^{\mathbf{y},(t)}=\mathbf{0}$.
\end{proof}

\begin{proof}[Proof of Lemma~\ref{lemm4}]
We have
$$
\begin{aligned}
\varepsilon_{t+1} 
& \leq(1+\beta) \mathbb{E}\left\|\hat{\mathbf{y}}^{(t)}-\overline{\mathbf{y}}^{(t+1)}\right\|^{2}+\left(1+\frac{1}{\beta}\right) \mathbb{E}\left\|\hat{\mathbf{y}}^{(t+1)}-\hat{\mathbf{y}}^{(t)}\right\|^{2} \\
& \leq(1+\beta)(1+\alpha)\left(1-K \eta^{\mathbf{y}} \mu\right) \varepsilon_{t} \\
&+(1+\beta)\left(1+\frac{1}{\alpha}\right)\left(\eta^{\mathbf{y}}\right)^{2} L^{2} K\left(\mathcal{E}_{t}^{\mathbf{x}}+\mathcal{E}_{t}^{\mathbf{y}}\right)+\left(1+\frac{1}{\beta}\right) \kappa^{2} \mathbb{E}\left\|\overline{\mathbf{x}}^{(t+1)}-\overline{\mathbf{x}}^{(t)}\right\|^{2}+(1+\beta) \frac{K\left(\eta^{\mathbf{y}}\right)^{2} \sigma^{2}}{n} \\
& \leq\left(1-\frac{K \eta^{\mathbf{y}} \mu}{3}\right) \varepsilon_{t}+\frac{6 \eta^{\mathbf{y}} L^{2}}{\mu}\left(\mathcal{E}_{t}^{\mathbf{x}}+\mathcal{E}_{t}^{\mathbf{y}}\right)+\frac{4 \eta^{\mathbf{y}} \sigma^{2}}{n \mu} \\
&+\frac{4 \kappa^{2}}{K \eta^{\mathbf{y}} \mu}\left(2 K\left(\eta^{\mathbf{x}}\right)^{2} L^{2}\left(\mathcal{E}_{t}^{\mathbf{x}}+\mathcal{E}_{t}^{\mathbf{y}}\right)+4 K^{2} L^{2}\left(\eta^{\mathbf{x}}\right)^{2} \varepsilon_{t}+4 K^{2}\left(\eta^{\mathbf{x}}\right)^{2} \mathbb{E}\|\nabla \Phi(\overline{\mathbf{x}}(t))\|^{2}+\frac{K\left(\eta^{\mathbf{x}}\right)^{2} \sigma^{2}}{n}\right) \\
&=\left(1-\frac{K \eta^{y} L}{3 \kappa}+\frac{16 L \kappa^{3} K\left(\eta^{\mathbf{x}}\right)^{2}}{\eta^{\mathbf{y}}}\right) \varepsilon_{t} \\
&+\left(\frac{8 L \kappa^{3}\left(\eta^{\mathbf{x}}\right)^{2}}{\eta^{y}}+6 \eta^{\mathbf{y}} L \kappa\right)\left(\mathcal{E}_{t}^{\mathbf{x}}+\mathcal{E}_{t}^{\mathbf{y}}\right)+\frac{16 \kappa^{3} K\left(\eta^{\mathbf{x}}\right)^{2}}{\eta^{y} L} \mathbb{E}\left\|\nabla \Phi(\overline{\mathbf{x}}^{(t)})\right\|^{2}+\frac{4 \kappa^{3}\left(\eta^{\mathbf{x}}\right)^{2} \sigma^{2}}{n \eta^{y} L}+\frac{4 \eta^{\mathbf{y}} \sigma^{2} \kappa}{n L}
\end{aligned}
$$
where we used the bound in Lemma~\ref{lemm10} for the first term, Proposition~\ref{prop1} for the second term, replaced $ \alpha = \beta = \frac{K\eta^{\mathbf{y}} \mu}{3} $ and used~\eqref{eq_proof_lemm9} in Lemma~\ref{lemm9}.
Seeing $\eta^{\mathbf{x}} \leq \frac{\eta^{y}}{4 \sqrt{6} \kappa^{2}}$ completes the proof.
\end{proof}

\subsection{Proof of Lemma~\ref{lemm5}}
\begin{proof}[Proof of Lemma~\ref{lemm5}]
Proposition~\ref{prop1} indicates that $\Phi(\cdot)$ is $2 \kappa L$-smooth, and hence yields
$$
\begin{aligned}
& \Phi(\overline{\mathbf{x}}^{(t+1)})=\Phi\left(\overline{\mathbf{x}}^{(t)}-\frac{\eta^{\mathbf{x}}}{n} \sum_{i, k}\left(\nabla_{\mathbf{x}} F_{i}\left(\mathbf{x}_{i}^{(t)+k}, \mathbf{y}_{i}^{(t)+k} ; \xi_{i}^{(t)+k}\right)+\mathbf{c}_{i}^{\mathbf{x}(t)}\right)\right) \\
& \leq \Phi(\overline{\mathbf{x}}^{(t)})+\underbrace{\left\langle\nabla \Phi(\overline{\mathbf{x}}^{(t)}),-\frac{\eta^{\mathbf{x}}}{n} \sum_{i, k}\left(\nabla_{\mathbf{x}} F_{i}\left(\mathbf{x}_{i}^{(t)+k}, \mathbf{y}_{i}^{(t)+k} ; \xi_{i}^{(t)+k}\right)+\mathbf{c}_{i}^{\mathbf{x},(t)}\right)\right\rangle}_{\equiv U}+\kappa L \mathbb{E}\left\|\overline{\mathbf{x}}^{(t+1)}-\overline{\mathbf{x}}^{(t)}\right\|^{2}
\end{aligned}
$$
Further derivation of an upper bound for $\mathbb{E}[U]$ as follows gives
$$
\begin{aligned}
\mathbb{E} & {[U] \equiv \mathbb{E}\left\langle\nabla \Phi(\overline{\mathbf{x}}^{(t)}),-\frac{\eta^{\mathbf{x}}}{n} \sum_{i, k}\left(\nabla_{\mathbf{x}} F_{i}\left(\mathbf{x}_{i}^{(t)+k}, \mathbf{y}_{i}^{(t)+k} ; \xi_{i}^{(t)+k}\right)+\mathbf{c}_{i}^{\mathbf{x},(t)}\right)\right\rangle } \\
= & \mathbb{E}\left\langle\nabla \Phi(\overline{\mathbf{x}}^{(t)}),-\frac{\eta^{\mathbf{x}}}{n} \sum_{i, k} \mathbb{E}_{\xi_{i}^{(t)+k}} \nabla_{\mathbf{x}} F_{i}\left(\mathbf{x}_{i}^{(t)+k}, \mathbf{y}_{i}^{(t)+k} ; \xi_{i}^{(t)+k}\right)\right\rangle \\
= & -\eta^{\mathbf{x}} \mathbb{E}\left\langle\nabla \Phi(\overline{\mathbf{x}}^{(t)}), \frac{1}{n} \sum_{i, k}\left(\nabla_{\mathbf{x}} f_{i}(\mathbf{x}_{i}^{(t)+k}, \mathbf{y}_{i}^{(t)+k})\right.\right.
\\
& \left.\left.-\nabla_{\mathbf{x}} f_{i}(\overline{\mathbf{x}}^{(t)}, \overline{\mathbf{y}}^{(t)})+\nabla_{\mathbf{x}} f_{i}(\overline{\mathbf{x}}^{(t)}, \overline{\mathbf{y}}^{(t)})-\nabla_{\mathbf{x}} f_{i}\left(\overline{\mathbf{x}}^{(t)}, \hat{\mathbf{y}}^{(t)}\right)+\nabla_{\mathbf{x}} f_{i}\left(\overline{\mathbf{x}}^{(t)}, \hat{\mathbf{y}}^{(t)}\right)\right)\right\rangle \\
= & -K \eta^{\mathbf{x}} \mathbb{E}\left\|\nabla \Phi(\overline{\mathbf{x}}^{(t)})\right\|^{2} \\
& -\frac{\eta^{\mathbf{x}}}{n} \sum_{i, k}\left\langle\nabla \Phi(\overline{\mathbf{x}}^{(t)}), \nabla_{\mathbf{x}} f_{i}(\mathbf{x}_{i}^{(t)+k}, \mathbf{y}_{i}^{(t)+k})-\nabla_{\mathbf{x}} f_{i}(\overline{\mathbf{x}}^{(t)}, \overline{\mathbf{y}}^{(t)})+\nabla_{\mathbf{x}} f_{i}(\overline{\mathbf{x}}^{(t)}, \overline{\mathbf{y}}^{(t)})-\nabla_{\mathbf{x}} f_{i}\left(\overline{\mathbf{x}}^{(t)}, \hat{\mathbf{y}}^{(t)}\right)\right\rangle \\
\leq & -\frac{K \eta^{\mathbf{x}}}{2} \mathbb{E}\left\|\nabla \Phi(\overline{\mathbf{x}}^{(t)})\right\|^{2} \\
& +\frac{\eta^{\mathbf{x}}}{n} \sum_{i, k}\left(\mathbb{E}\left\|\nabla_{\mathbf{x}} f_{i}(\mathbf{x}_{i}^{(t)+k}, \mathbf{y}_{i}^{(t)+k})-\nabla_{\mathbf{x}} f_{i}(\overline{\mathbf{x}}^{(t)}, \overline{\mathbf{y}}^{(t)})\right\|^{2}+\mathbb{E}\left\|\nabla_{\mathbf{x}} f_{i}(\overline{\mathbf{x}}^{(t)}, \overline{\mathbf{y}}^{(t)})-\nabla_{\mathbf{x}} f_{i}\left(\overline{\mathbf{x}}^{(t)}, \hat{\mathbf{y}}^{(t)}\right)\right\|^{2}\right) \\
\leq & -\frac{K \eta^{\mathbf{x}}}{2} \mathbb{E}\left\|\nabla \Phi(\overline{\mathbf{x}}^{(t)})\right\|^{2}+\eta^{\mathbf{x}} L^{2}\left(\mathcal{E}_{t}^{\mathbf{x}}+\mathcal{E}_{t}^{\mathbf{y}}\right)+K \eta^{\mathbf{x}} L^{2} \varepsilon_{t}
\end{aligned}
$$
Applying the above upper bound for $\mathbb{E}[U]$ and~\eqref{eq_proof_lemm9} in the proof of Lemma~\ref{lemm9} gives
$$
\begin{aligned}
& \mathbb{E}\left[\Phi(\overline{\mathbf{x}}^{(t+1)})-\Phi(\overline{\mathbf{x}}^{(t)})\right] \leq \eta^{\mathbf{x}} L^{2}\left(\mathcal{E}_{t}^{\mathbf{x}}+\mathcal{E}_{t}^{\mathbf{y}}\right)+L^{2} \eta^{\mathbf{x}} K \varepsilon_{t}-\frac{\eta^{\mathbf{x}} K}{2} \mathbb{E}\left\|\nabla \Phi(\overline{\mathbf{x}}^{(t)})\right\|^{2}+\kappa L \mathbb{E}\left\|\overline{\mathbf{x}}^{(t+1)}-\overline{\mathbf{x}}^{(t)}\right\|^{2} \\
& \leq \mathbb{E} \Phi(\overline{\mathbf{x}}^{(t)})+\left(\eta^{\mathbf{x}} L^{2}+2 K\left(\eta^{\mathbf{x}}\right)^{2} L^{3} \kappa\right)\left(\mathcal{E}_{t}^{\mathbf{x}}+\mathcal{E}_{t}^{\mathbf{y}}\right) \\
& \\
& \quad+\frac{K\left(\eta^{\mathbf{x}}\right)^{2} L \kappa \sigma^{2}}{n}+\left(L^{2} \eta^{\mathbf{x}} K+4 K^{2} L^{3}\left(\eta^{\mathbf{x}}\right)^{2} \kappa\right) \varepsilon_{t}+\left(4 K^{2}\left(\eta^{\mathbf{x}}\right)^{2} L \kappa-\frac{\eta^{\mathbf{x}} K}{2}\right) \mathbb{E} \| \nabla \phi\left(\left(\overline{\mathbf{x}}^{(t)}\right) \|^{2}\right.
\end{aligned}
$$
Using $\eta^{\mathbf{x}} \leq \frac{1}{16 K L \kappa}$ completes the proof.
\end{proof}

\subsection{Proof of Lemma~\ref{lemm6}}\label{sec_proof,lemm6}
\begin{proof}[Proof of Lemma~\ref{lemm6}]
According to the Lemma~\ref{lemm1}, we have
$$
\begin{aligned}
0 \leq-E^{\mathrm{x}} \frac{L}{K} \eta_{c}^{\mathrm{y}} \mathcal{E}_{t}^{\mathrm{x}} & +3 E^{\mathrm{x}} L \eta_{c}^{y} \Xi_{t}^{\mathrm{x}}+12 E^{\mathrm{x}} K\left(\eta_{c}^{\mathrm{x}}\right)^{2} \eta_{c}^{y} L^{3} \mathcal{E}_{t}^{\mathrm{y}}+12 E^{\mathrm{x}} K^{2}\left(\eta_{c}^{\mathrm{x}}\right)^{2} \eta_{c}^{\mathrm{y}} L^{3} \gamma_{t}^{\mathrm{x}} \\
& +12 E^{\mathrm{x}} K^{2}\left(\eta_{c}^{\mathrm{x}}\right)^{2} \eta_{c}^{\mathrm{y}} L^{3} \varepsilon_{t}+12 E^{\mathrm{x}} K^{2}\left(\eta_{c}^{\mathrm{x}}\right)^{2} \eta_{c}^{\mathrm{y}} L \mathbb{E}\left\|\nabla \Phi(\overline{\mathbf{x}}^{(t)})\right\|^{2}+3 E^{\mathrm{x}} K\left(\eta_{c}^{\mathrm{x}}\right)^{2} \eta_{c}^{y} L \sigma^{2} \\
0 \leq-E^{\mathrm{y}} \frac{L}{K} \eta_{c}^{\mathrm{y}} \mathcal{E}_{t}^{\mathrm{y}} & +3 E^{\mathrm{y}} L \eta_{c}^{y} \Xi_{t}^{\mathrm{y}}+12 E^{\mathrm{y}} K\left(\eta_{c}^{\mathrm{y}}\right)^{3} L^{3} \mathcal{E}_{t}^{\mathrm{x}}+12 E^{\mathrm{y}} K^{2}\left(\eta_{c}^{\mathrm{y}}\right)^{3} L^{3} \gamma_{t}^{\mathrm{y}}+6 E^{\mathrm{y}} K^{2}\left(\eta_{c}^{\mathrm{y}}\right)^{3} L^{3} \varepsilon_{t}+3 E^{\mathrm{y}} K\left(\eta_{c}^{y}\right)^{3} L \sigma^{2}
\end{aligned}
$$
Applying the definition of $\mathscr{H}_{t}$ as in~\eqref{eq_lemm6m} and using the aformentioned display along with Lemmas~\ref{lemm1}---\ref{lemm5}, we have
$$
\begin{aligned}
& \mathscr{H}_{t+1}-\mathscr{H}_{t} \\
& =\mathbb{E}\left[\Phi(\overline{\mathbf{x}}^{(t+1)})-\Phi(\overline{\mathbf{x}}^{(t)})\right] \\
& +B^{\mathbf{x}} \eta_{c}^{y} L\left(\Xi_{t+1}^{\mathrm{x}}-\Xi_{t}^{\mathrm{x}}\right)+B^{\mathrm{y}} \eta_{c}^{\mathrm{y}} L\left(\Xi_{t+1}^{y}-\Xi_{t}^{\mathrm{y}}\right) \\
& +A^{\times} K^{2} L^{3}\left(\eta_{t}^{y}\right)^{3}\left(\gamma_{t+1}^{\mathrm{x}}-\gamma_{t}^{\mathrm{x}}\right)+A^{\mathrm{y}} K^{2} L^{3}\left(\eta_{c}^{\mathrm{y}}\right)^{3}\left(\gamma_{t+1}^{\mathrm{y}}-\gamma_{t}^{\mathrm{y}}\right) \\
& +C \frac{1}{K \kappa p}\left(\varepsilon_{t+1}-\varepsilon_{t}\right)+0+0 \\
& \leq 2 \eta^{\mathbf{x}} L^{2}\left(\mathcal{E}_{t}^{\mathbf{x}}+\mathcal{E}_{t}^{\mathbf{y}}\right)+2 L^{2} \eta^{\mathbf{x}} K \varepsilon_{t}-\frac{\eta^{\times} K}{4} \mathbb{E}\left\|\nabla \Phi(\overline{\mathbf{x}}^{(t)})\right\|^{2}+\frac{K\left(\eta^{\times}\right)^{2} L \sigma^{2} \kappa}{n} \\
& +B^{\mathbf{x}} \eta_{c}^{y} L\left(-\frac{p}{2} \Xi_{t}^{\mathbf{x}}+\frac{6 K\left(\eta^{\mathbf{x}}\right)^{2} L^{2}}{p}\left(\mathcal{E}_{t}^{\mathbf{x}}+\mathcal{E}_{t}^{y}\right)+\frac{6 K^{2}\left(\eta^{\mathbf{x}}\right)^{2} L^{2}}{p} \gamma_{t}^{\mathbf{x}}+K\left(\eta^{\mathbf{x}}\right)^{2} \sigma^{2}\right) \\
& +B^{\mathbf{y}} \eta_{c}^{y} L\left(-\frac{p}{2} \Xi_{t}^{\mathbf{y}}+\frac{6 K\left(\eta^{\mathbf{y}}\right)^{2} L^{2}}{p}\left(\mathcal{E}_{t}^{\mathbf{x}}+\mathcal{E}_{t}^{y}\right)+\frac{6 K^{2}\left(\eta^{\mathbf{y}}\right)^{2} L^{2}}{p} \gamma_{t}^{y}+K\left(\eta^{\mathbf{y}}\right)^{2} \sigma^{2}\right) \\
& +A^{\times} K^{2} L^{3}\left(\eta_{c}^{y}\right)^{3}\left(-\frac{p}{2} \gamma_{t}^{\mathbf{x}}+\frac{30}{p K}\left(\mathcal{E}_{t}^{\mathbf{x}}+\mathcal{E}_{t}^{y}\right)+\frac{12 K^{2} L^{2}}{p}\left(2\left(\eta^{\mathbf{x}}\right)^{2}+\left(\eta^{y}\right)^{2}\right) \varepsilon_{t}+\frac{24 K^{2}\left(\eta^{\mathbf{x}}\right)^{2}}{p} \mathbb{E}\left\|\nabla \Phi(\overline{\mathbf{x}}^{(t)})\right\|^{2}+\frac{2 \sigma^{2}}{K L^{2}}\right) \\
& +A^{\mathrm{y}} K^{2} L^{3}\left(\eta_{c}^{y}\right)^{3}\left(-\frac{p}{2} \gamma_{t}^{y}+\frac{30}{p K}\left(\mathcal{E}_{t}^{\mathbf{x}}+\mathcal{E}_{t}^{y}\right)+\frac{12 K^{2} L^{2}}{p}\left(2\left(\eta^{\mathbf{x}}\right)^{2}+\left(\eta^{y}\right)^{2}\right) \varepsilon_{t}+\frac{24 K^{2}\left(\eta^{\mathbf{x}}\right)^{2}}{p} \mathbb{E}\left\|\nabla \Phi(\overline{\mathbf{x}}^{(t)})\right\|^{2}+\frac{2 \sigma^{2}}{K L^{2}}\right) \\
& +C \frac{1}{K \kappa p}\left(-\frac{K \eta^{y} L}{6 \kappa} \varepsilon_{t}+12 \eta^{y} L \kappa\left(\mathcal{E}_{t}^{x}+\mathcal{E}_{t}^{y}\right)+\frac{16 \kappa^{3} K\left(\eta^{x}\right)^{2}}{\eta^{y} L} \mathbb{E}\left\|\nabla \Phi(\overline{\mathbf{x}}^{(t)})\right\|^{2}+\frac{8 \eta^{y} \sigma^{2} \kappa}{n L}\right) \\
& -E^{\mathbf{x}} \frac{L}{K} \eta_{c}^{\mathrm{y}} \mathcal{E}_{t}^{\mathbf{x}}+3 E^{\mathbf{x}} L \eta_{c}^{\mathbf{y}} \Xi_{t}^{\mathbf{x}}+12 E^{\mathbf{x}} K\left(\eta_{c}^{\mathbf{x}}\right)^{2} \eta_{c}^{y} L^{3} \mathcal{E}_{t}^{\mathbf{y}}+12 E^{\mathbf{x}} K^{2}\left(\eta_{c}^{\mathbf{x}}\right)^{2} \eta_{c}^{\mathbf{y}} L^{3} \gamma_{t}^{\mathbf{x}} \\
& +12 E^{\mathrm{x}} K^{2}\left(\eta_{c}^{\mathrm{x}}\right)^{2} \eta_{c}^{\mathrm{y}} L^{3} \varepsilon_{t}+12 E^{\mathrm{x}} K^{2}\left(\eta_{c}^{\mathbf{x}}\right)^{2} \eta_{c}^{\mathrm{y}} L \mathbb{E}\left\|\nabla \Phi(\overline{\mathbf{x}}^{(t)})\right\|^{2}+3 E^{\mathrm{x}} K\left(\eta_{c}^{\mathrm{x}}\right)^{2} \eta_{c}^{\mathrm{y}} L \sigma^{2} \\
& -E^{\mathbf{y}} \frac{L}{K} \eta_{c}^{y} \mathcal{E}_{t}^{\mathbf{y}}+3 E^{y} L \eta_{c}^{y} \Xi_{t}^{\mathrm{y}}+12 E^{\mathrm{y}} K\left(\eta_{c}^{y}\right)^{3} L^{3} \mathcal{E}_{t}^{\mathbf{x}}+12 E^{\mathrm{y}} K^{2}\left(\eta_{c}^{y}\right)^{3} L^{3} \gamma_{t}^{\mathbf{y}}+6 E^{\mathrm{y}} K^{2}\left(\eta_{c}^{y}\right)^{3} L^{3} \varepsilon_{t}+3 E^{\mathrm{y}} K\left(\eta_{c}^{y}\right)^{3} L \sigma^{2}
\end{aligned}
$$
Further rearranging gives
$$
\begin{aligned}
\mathscr{H}_{t+1}-\mathscr{H}_{t} \leq & C_{1}^{\mathrm{x}} \cdot\left(\eta_{c}^{\mathrm{y}}\right)^{3} K^{2} L^{3} \gamma_{t}^{\mathrm{x}}+C_{1}^{\mathrm{y}} \cdot\left(\eta_{c}^{\mathrm{y}}\right)^{3} K^{2} L^{3} \gamma_{t}^{\mathrm{y}}+C_{2}^{\mathrm{x}} \cdot \Xi_{t}^{\mathrm{x}} \eta_{c}^{\mathrm{y}} L+C_{2}^{\mathrm{y}} \cdot \Xi_{t}^{\mathrm{y}} \eta_{c}^{\mathrm{y}} L+\frac{L}{K} \eta_{c}^{\mathrm{y}} \mathcal{E}_{t}^{\mathrm{x}} \\
& +C_{3}^{\mathrm{y}} \cdot \frac{L}{K} \eta_{c}^{\mathrm{y}} \mathcal{E}_{t}^{\mathrm{y}}+C_{4}^{\varepsilon} \cdot \frac{L \eta_{c}^{\mathrm{y}}}{\kappa^{2}} \varepsilon_{t}+C_{4}^{\mathrm{x}} \cdot K \eta^{\mathbf{x}} \mathbb{E}\left\|\nabla \Phi(\overline{\mathbf{x}}^{(t)})\right\|^{2} \\
& +\frac{C_{4}^{\mathrm{y}}}{p} \cdot K L\left(\eta_{c}^{\mathrm{y}}\right)^{3} \sigma^{2}+\frac{K\left(\eta^{\mathbf{x}}\right)^{2} L \kappa}{n} \sigma^{2}+C \frac{8 \eta^{\mathbf{y}}}{n L K p} \sigma^{2}
\end{aligned}
$$
where constants $C_{1}^{\mathrm{x}}, C_{1}^{\mathrm{y}}, C_{2}^{\mathrm{x}}, C_{2}^{\mathrm{y}}, C_{3}^{\mathrm{x}}, C_{4}^{\varepsilon}, C_{4}^{\mathrm{x}}$ and $C_{4}^{\mathrm{y}}$ are chosen such that
$$
\begin{aligned}
& -A^{\mathbf{x}} \frac{p}{2}+B^{\mathbf{x}} \frac{6\left(\eta_{s}^{\mathbf{x}}\right)^{2}}{p}+12 E^{\mathbf{x}} \leq C_{1}^{\mathbf{x}} \\
& -A^{\mathrm{y}} \frac{p}{2}+B^{\mathrm{y}} \frac{6\left(\eta_{s}^{\mathrm{y}}\right)^{2}}{p}+12 E^{\mathrm{y}} \leq C_{1}^{y} \\
& -B^{\mathbf{x}} \frac{p}{2}+3 E^{\mathbf{x}} \leq C_{2}^{\mathbf{x}} \\
& -B^{\mathrm{y}} \frac{p}{2}+3 E^{\mathrm{y}} \leq C_{2}^{\mathrm{y}} \\
& -E^{\mathrm{x}}+B^{\mathrm{x}} \frac{6 K^{2} L^{2}\left(\eta^{\mathbf{x}}\right)^{2}}{p}+B^{\mathrm{y}} \frac{6 K^{2} L^{2}\left(\eta^{\mathbf{y}}\right)^{2}}{p}+A^{\times} \frac{30\left(\eta_{c}^{y}\right)^{2} L^{2} K^{2}}{p}+A^{\mathrm{y}} \frac{30\left(\eta_{c}^{\mathrm{y}}\right)^{2} L^{2} K^{2}}{p} \\
& +12 E^{\mathrm{y}} K^{2} L^{2}\left(\eta_{c}^{\mathrm{y}}\right)^{2}+2 \eta^{\mathbf{x}} K L+C \frac{12 \eta_{s}^{y}}{p} \leq C_{3}^{\mathrm{x}} \\
& -C \frac{\eta_{s}^{y}}{6 p}+A^{\times} \frac{12 K^{4} L^{4}}{p}\left(\eta_{c}^{y}\right)^{2} \cdot 3\left(\eta^{\mathbf{y}}\right)^{2} \kappa^{2}+A^{\mathrm{y}} \frac{12 K^{4} L^{4}}{p}\left(\eta_{c}^{\mathrm{y}}\right)^{2} \cdot 3\left(\eta^{\mathbf{y}}\right)^{2} \kappa^{2} \\
& +12 E^{\mathrm{x}} K^{2} L^{2}\left(\eta_{c}^{\mathrm{x}}\right)^{2} \kappa^{2}+6 E^{\mathrm{y}} K^{2} L^{2}\left(\eta_{c}^{y}\right)^{2} \kappa^{2}+2 L K \kappa^{2} \frac{\eta^{\mathbf{x}}}{\eta_{c}^{y}} \leq C_{4}^{\varepsilon} \\
& -\frac{1}{4}+A^{\times} \frac{24 K^{3} L^{3}}{p}\left(\eta_{c}^{y}\right)^{3} \eta^{\mathbf{x}}+A^{\mathrm{y}} \frac{24 K^{3} L^{3}}{p}\left(\eta_{c}^{y}\right)^{3} \eta^{\mathbf{x}}+C \frac{16 \kappa^{2} \eta^{\mathbf{x}}}{K L \eta^{y} p} \\
& +12 E^{\mathrm{x}} K L \eta_{c}^{\mathrm{y}} \frac{\eta_{c}^{\mathrm{x}}}{\eta_{s}^{\mathrm{x}}} \leq C_{4}^{\mathrm{x}} \\
& B^{\mathrm{x}}\left(\eta_{s}^{\mathrm{x}}\right)^{2}+B^{\mathrm{y}}\left(\eta_{s}^{\mathrm{y}}\right)^{2}+2 A^{\mathrm{x}}+2 A^{\mathrm{y}}+3 E^{\mathrm{x}}+3 E^{\mathrm{y}} \leq \frac{C_{4}^{\mathrm{y}}}{p}
\end{aligned}
$$
By letting $E^{\mathbf{x}}=E^{\mathbf{y}}=v, \eta_{c}^{\mathrm{y}} \leq \frac{p}{300 v \kappa K L}, \eta_{c}^{\mathbf{x}} \leq \frac{\eta_{c}^{y}}{k^{2}}, \eta_{s}^{\mathbf{x}}=\eta_{s}^{\mathrm{y}}=p v, B^{\mathbf{x}}=B^{\mathbf{y}}=\frac{6 v}{p}, A^{\mathrm{x}}=A^{\mathrm{y}}=\frac{1}{p}\left(72 v^{3}+24 v\right)$, and $C=\frac{1}{24}$, there exists a global constant $v>1$ that ensures $C_{1}^{\mathrm{x}}, C_{1}^{\mathrm{y}}, C_{2}^{\mathrm{x}}, C_{2}^{\mathrm{y}}, C_{3}^{\mathrm{x}}, C_{3}^{\mathrm{y}}, C_{4}^{\varepsilon} \leq 0, C_{4}^{\mathrm{x}}<0$ and $C_{4}^{y} \geq 0$, hence proving Lemma~\ref{lemm6}.
\end{proof}

\section{Toolbox Lemmas}\label{secB}
We introduce some technical lemmas we will use from time to time in this paper.
Proofs are either standard or can be found in the provided references.

\begin{proposition}[\cite{lin2020gradient}]\label{prop1}
Under Assumption~\ref{assu2}, $\Phi(\cdot)$ is $L(1+\kappa)$-smooth.
Furthermore, $\mathbf{y}^{*}(\cdot)=\arg \max _{\mathbf{y} \in \mathbb{R}^{d_{y}}} f(\cdot, \mathbf{y})$ is $\kappa$-Lipschitz in the sense that for any $\mathbf{x}, \mathbf{x}^{\prime} \in \mathbb{R}^{d_{x}}$
$$
\left\|\mathbf{y}^{*}(\mathbf{x})-\mathbf{y}^{*}\left(\mathbf{x}^{\prime}\right)\right\| \leq \kappa\|\mathbf{x}-\mathbf{x}^{\prime}\|
$$
\end{proposition}

\begin{proposition}[\cite{bubeck2015convex}]\label{prop2}
Under Assumption~\ref{assu2}, for each $\mathbf{x} \in \mathbb{R}^{d_{x}}$ and $\mathbf{y}, \mathbf{y}^{\prime} \in \mathbb{R}^{d_{y}}$, we have
$$
\nabla_{\mathbf{y}} f(\mathbf{x}, \mathbf{y})^{\top}\left(\mathbf{y}-\mathbf{y}^{\prime}\right)+\frac{1}{2 L}\left\|\nabla_{\mathbf{y}} f(\mathbf{x}, \mathbf{y})\right\|^{2}+\frac{\mu}{2}\|\mathbf{y}-\mathbf{y}^{\prime}\|^{2} \leq f\left(\mathbf{x}, \mathbf{y}^{+}\right)-f\left(\mathbf{x}, \mathbf{y}^{\prime}\right)
$$
where $\mathbf{y}^{+} \equiv \mathbf{y}-\frac{1}{L} \nabla_{\mathbf{y}} f(\mathbf{x}, \mathbf{y})$.
\end{proposition}

\begin{lemma}\label{lemm11}
For a set of arbitrary vectors $a_{1}, \ldots, a_{n}$ such that $a_{i} \in \mathbb{R}^{d_{x}}$, we have
$$
\left\|\frac{1}{n} \sum_{i=1}^{n} a_{i}\right\|^{2} \leq \frac{1}{n} \sum_{i=1}^{n}\left\|a_{i}\right\|^{2}
$$
\end{lemma}

\begin{lemma}[Young's + Cauchy-Schwarz Inequality]\label{lemm12}
For any vectors $a, b \in \mathbb{R}^{d_{x}}$ and $\alpha>0$ we have
$$
2\langle a, b\rangle \leq \alpha\|a\|^{2}+\frac{1}{\alpha}\|b\|^{2}
$$
and
$$
\|a+b\|^{2} \leq(1+\alpha)\|a\|^{2}+\left(1+\frac{1}{\alpha}\right)\|b\|^{2}
$$
\end{lemma}


\end{document}

%% file: macros.tex
\long\def\comment#1{}

\usepackage{stackrel}

\newcommand{\RKHS}{\ensuremath{\mathscr{F}}}

\newcommand{\StateSpgoodh}[1]{\ensuremath{}}

\newcommand{\StateSpbadh}[1]{\ensuremath{}}

\newcommand{\Data}{\ensuremath{\mathcal{D}}}

\usepackage{upgreek} \newcommand{\distr}{\ensuremath{\upmu}}

\newcommand{\distrdata}{\ensuremath{\bar{\distr}}}

\newcommand{\sdistrdata}{\ensuremath{\sdistr_{\Data}}}

\newcommand{\sdistr}{\ensuremath{\xi}}

\DeclarePairedDelimiterX{\anglep}[1]{(}{)}{#1}
\makeatletter
\newcommand{\@spanstar}[1]{{\rm span}\anglep*{#1}}
\newcommand{\@spannostar}[2][]{{\rm span}\anglep[#1]{#2}}
\newcommand{\Span}{\@ifstar\@spanstar\@spannostar}
\makeatother

\DeclarePairedDelimiterX{\dfun}[2]{(}{)}{#1 \;\delimsize\|\; #2}
\makeatletter
\newcommand{\@trunstar}[2]{\chi^2\dfun*{#1}{#2}}
\newcommand{\@trunnostar}[3][]{\chi^2\dfun[#1]{#2}{#3}}
\newcommand{\chisq}{\@ifstar\@trunstar\@trunnostar}
\makeatother

\DeclarePairedDelimiterX{\inprod}[2]{\langle}{\rangle}{#1, \, #2}

\DeclarePairedDelimiterX{\kulldiv}[2]{(}{)}{#1\;\delimsize\|\;#2}
\makeatletter
\newcommand{\@kullstar}[2]{D_{\text{KL}}\kulldiv*{#1}{#2}}
\newcommand{\@kullnostar}[3][]{D_{\text{KL}}\kulldiv[#1]{#2}{#3}}
\newcommand{\kull}{\@ifstar\@kullstar\@kullnostar}
\makeatother

\makeatletter
\newcommand{\@hilinstar}[2]{\inprod*{#1}{#2}_{\RKHS}}
\newcommand{\@hilinnostar}[3][]{\inprod[#1]{#2}{#3}_{\RKHS}}
\newcommand{\hilin}{\@ifstar\@hilinstar\@hilinnostar}
\makeatother

\makeatletter
\newcommand{\@mudatainstar}[2]{\inprod*{#1}{#2}_{\distrdata}}
\newcommand{\@mudatainnostar}[3][]{\inprod[#1]{#2}{#3}_{\distrdata}}
\newcommand{\mudatain}{\@ifstar\@mudatainstar\@mudatainnostar}
\makeatother

\makeatletter
\newcommand{\@mudatahinstar}[3]{\inprod*{#2}{#3}_{\distrdata}}
\newcommand{\@mudatahinnostar}[4][]{\inprod[#1]{#3}{#4}_{\distrdata}}
\newcommand{\mudatahin}{\@ifstar\@mudatahinstar\@mudatahinnostar}
\makeatother

\DeclarePairedDelimiterX{\defabs}[1]{|}{|}{#1}
\makeatletter
\newcommand{\@absstar}[1]{\defabs*{#1}}
\newcommand{\@absnostar}[2][]{\defabs[#1]{#2}}
\newcommand{\abs}{\@ifstar\@absstar\@absnostar}
\makeatother

\DeclarePairedDelimiterX{\norm}[1]{\|}{\|}{#1}
\makeatletter
\newcommand{\@normstar}[1]{\norm*{#1}_{\RKHS}}
\newcommand{\@normnostar}[2][]{\norm[#1]{#2}_{\RKHS}}
\newcommand{\hilnorm}{\@ifstar\@normstar\@normnostar}
\makeatother

\DeclareFontFamily{U}{matha}{\hyphenchar\font45}
\DeclareFontShape{U}{matha}{m}{n}{
	<-6> matha5 <6-7> matha6 <7-8> matha7
	<8-9> matha8 <9-10> matha9
	<10-12> matha10 <12-> matha12
}{}
\DeclareSymbolFont{matha}{U}{matha}{m}{n}

\DeclareFontFamily{U}{mathx}{\hyphenchar\font45}
\DeclareFontShape{U}{mathx}{m}{n}{
	<-6> mathx5 <6-7> mathx6 <7-8> mathx7
	<8-9> mathx8 <9-10> mathx9
	<10-12> mathx10 <12-> mathx12
}{}
\DeclareSymbolFont{mathx}{U}{mathx}{m}{n}

\DeclareMathDelimiter{\vvvert} {0}{matha}{"7E}{mathx}{"17}%

\DeclarePairedDelimiterX{\opnorm}[1]{\vvvert}{\vvvert}{#1}

\makeatletter
\newcommand{\@hilopnormstar}[1]{\opnorm*{#1}_{\RKHS}}
\newcommand{\@hilopnormnostar}[2][]{\opnorm[#1]{#2}_{\RKHS}}
\newcommand{\hilopnorm}{\@ifstar\@hilopnormstar\@hilopnormnostar}
\makeatother

\makeatletter
\newcommand{\@muopnormstar}[1]{\opnorm*{#1}_{\distr}}
\newcommand{\@muopnormnostar}[2][]{\opnorm[#1]{#2}_{\distr}}
\newcommand{\muopnorm}{\@ifstar\@muopnormstar\@muopnormnostar}
\makeatother

\makeatletter
\newcommand{\@mudataopnormstar}[1]{\opnorm*{#1}_{\distrdata}}
\newcommand{\@mudataopnormnostar}[2][]{\opnorm[#1]{#2}_{\distrdata}}
\newcommand{\mudataopnorm}{\@ifstar\@mudataopnormstar\@mudataopnormnostar}
\makeatother

\makeatletter
\newcommand{\@supnormstar}[1]{\norm*{#1}_{\infty}}
\newcommand{\@supnormnostar}[2][]{\norm[#1]{#2}_{\infty}}
\newcommand{\supnorm}{\@ifstar\@supnormstar\@supnormnostar}
\makeatother

\makeatletter
\newcommand{\@munormstar}[1]{\norm*{#1}_{\distr}}
\newcommand{\@munormnostar}[2][]{\norm[#1]{#2}_{\distr}}
\newcommand{\munorm}{\@ifstar\@munormstar\@munormnostar}
\makeatother

\makeatletter
\newcommand{\@mudatanormstar}[1]{\norm*{#1}_{\distrdata}}
\newcommand{\@mudatanormnostar}[2][]{\norm[#1]{#2}_{\distrdata}}
\newcommand{\mudatanorm}{\@ifstar\@mudatanormstar\@mudatanormnostar}
\makeatother

\makeatletter
\newcommand{\@xidatanormstar}[1]{\norm*{#1}_{\sdistrdata}}
\newcommand{\@xidatanormnostar}[2][]{\norm[#1]{#2}_{\sdistrdata}}
\newcommand{\xidatanorm}{\@ifstar\@xidatanormstar\@xidatanormnostar}
\makeatother

\makeatletter
\newcommand{\@distrnormstar}[2]{\norm*{#1}_{#2}}
\newcommand{\@distrnormnostar}[3][]{\norm[#1]{#2}_{#3}}
\newcommand{\distrnorm}{\@ifstar\@distrnormstar\@distrnormnostar}
\makeatother

\makeatletter
\newcommand{\@psinormstar}[2]{\norm*{#2}_{\psi_{#1}}}
\newcommand{\@psinormnostar}[3][]{\norm[#1]{#3}_{\psi_{#2}}}
\newcommand{\psinorm}{\@ifstar\@psinormstar\@psinormnostar}
\makeatother

\newcommand{\featureh}[1]{\ensuremath{\phi}}

\newcommand{\Lip}{\ensuremath{L}}

\newcommand{\Lipf}[1]{\ensuremath{\Lip_{f}}}

\newenvironment{carlist}
{\begin{list}{$\bullet$}
		{\setlength{\topsep}{0.1in} \setlength{\partopsep}{0in}
			\setlength{\parsep}{0.1in} \setlength{\itemsep}{\parskip}
			\setlength{\leftmargin}{0.15in} \setlength{\rightmargin}{0.08in}
			\setlength{\listparindent}{0in} \setlength{\labelwidth}{0.08in}
			\setlength{\labelsep}{0.1in} \setlength{\itemindent}{0in}}}
	{\end{list}}

\newcommand{\bcar}{\begin{carlist}}
	\newcommand{\ecar}{\end{carlist}}

%% file: main_submission.bbl
\newcommand{\etalchar}[1]{$^{#1}$}
\begin{thebibliography}{GPAM{\etalchar{+}}20}

\bibitem[B{\etalchar{+}}15]{bubeck2015convex}
S{\'e}bastien Bubeck et~al.
\newblock Convex optimization: Algorithms and complexity.
\newblock {\em Foundations and Trends{\textregistered} in Machine Learning},
  8(3-4):231--357, 2015.

\bibitem[CYL24]{chen2024efficient}
Lesi Chen, Haishan Ye, and Luo Luo.
\newblock An efficient stochastic algorithm for decentralized
  nonconvex-strongly-concave minimax optimization.
\newblock In {\em International Conference on Artificial Intelligence and
  Statistics}, pages 1990--1998. PMLR, 2024.

\bibitem[GPAM{\etalchar{+}}20]{goodfellow2020generative}
Ian Goodfellow, Jean Pouget-Abadie, Mehdi Mirza, Bing Xu, David Warde-Farley,
  Sherjil Ozair, Aaron Courville, and Yoshua Bengio.
\newblock Generative adversarial networks.
\newblock {\em Communications of the ACM}, 63(11):139--144, 2020.

\bibitem[HPMG20]{hsieh2020non}
Kevin Hsieh, Amar Phanishayee, Onur Mutlu, and Phillip Gibbons.
\newblock The non-iid data quagmire of decentralized machine learning.
\newblock In {\em International Conference on Machine Learning}, pages
  4387--4398. PMLR, 2020.

\bibitem[KLB{\etalchar{+}}20]{koloskova2020unified}
Anastasia Koloskova, Nicolas Loizou, Sadra Boreiri, Martin Jaggi, and Sebastian
  Stich.
\newblock A unified theory of decentralized {SGD} with changing topology and
  local updates.
\newblock In {\em International Conference on Machine Learning}, pages
  5381--5393. PMLR, 2020.

\bibitem[KLS21]{koloskova2021improved}
Anastasiia Koloskova, Tao Lin, and Sebastian~U Stich.
\newblock An improved analysis of gradient tracking for decentralized machine
  learning.
\newblock {\em Advances in Neural Information Processing Systems},
  34:11422--11435, 2021.

\bibitem[KMA{\etalchar{+}}21]{kairouz2021advances}
Peter Kairouz, H~Brendan McMahan, Brendan Avent, Aur{\'e}lien Bellet, Mehdi
  Bennis, Arjun~Nitin Bhagoji, Kallista Bonawitz, Zachary Charles, Graham
  Cormode, Rachel Cummings, et~al.
\newblock Advances and open problems in federated learning.
\newblock {\em Foundations and trends{\textregistered} in machine learning},
  14(1--2):1--210, 2021.

\bibitem[LCCC20]{li2020communication}
Boyue Li, Shicong Cen, Yuxin Chen, and Yuejie Chi.
\newblock Communication-efficient distributed optimization in networks with
  gradient tracking and variance reduction.
\newblock {\em Journal of Machine Learning Research}, 21(180):1--51, 2020.

\bibitem[LJJ20]{lin2020gradient}
Tianyi Lin, Chi Jin, and Michael Jordan.
\newblock On gradient descent ascent for nonconvex-concave minimax problems.
\newblock In {\em International Conference on Machine Learning}, pages
  6083--6093. PMLR, 2020.

\bibitem[LLKS24]{liu2024decentralized}
Yue Liu, Tao Lin, Anastasia Koloskova, and Sebastian~U Stich.
\newblock Decentralized gradient tracking with local steps.
\newblock {\em Optimization Methods and Software}, pages 1--28, 2024.

\bibitem[LYHZ20]{luo2020stochastic}
Luo Luo, Haishan Ye, Zhichao Huang, and Tong Zhang.
\newblock Stochastic recursive gradient descent ascent for stochastic
  nonconvex-strongly-concave minimax problems.
\newblock {\em Advances in Neural Information Processing Systems},
  33:20566--20577, 2020.

\bibitem[LZLL23]{liu2023precision}
Zhuqing Liu, Xin Zhang, Songtao Lu, and Jia Liu.
\newblock {PRECISION}: Decentralized constrained min-max learning with low
  communication and sample complexities.
\newblock In {\em Proceedings of the Twenty-fourth International Symposium on
  Theory, Algorithmic Foundations, and Protocol Design for Mobile Networks and
  Mobile Computing}, pages 191--200, 2023.

\bibitem[NAYU23]{nguyen2023performance}
Edward Duc~Hien Nguyen, Sulaiman~A Alghunaim, Kun Yuan, and C{\'e}sar~A Uribe.
\newblock On the performance of gradient tracking with local updates.
\newblock In {\em 2023 62nd IEEE Conference on Decision and Control (CDC)},
  pages 4309--4313. IEEE, 2023.

\bibitem[RBD{\etalchar{+}}24]{rogozin2024decentralized}
Alexander Rogozin, Aleksandr Beznosikov, Darina Dvinskikh, Dmitry Kovalev,
  Pavel Dvurechensky, and Alexander Gasnikov.
\newblock Decentralized saddle point problems via non-euclidean mirror prox.
\newblock {\em Optimization Methods and Software}, pages 1--26, 2024.

\bibitem[SPJ23]{sharma2023federated}
Pranay Sharma, Rohan Panda, and Gauri Joshi.
\newblock Federated minimax optimization with client heterogeneity.
\newblock {\em Transactions on machine learning research}, 2023.

\bibitem[SPJV22]{sharma2022federated}
Pranay Sharma, Rohan Panda, Gauri Joshi, and Pramod Varshney.
\newblock Federated minimax optimization: Improved convergence analyses and
  algorithms.
\newblock In {\em International Conference on Machine Learning}, pages
  19683--19730. PMLR, 2022.

\bibitem[SW22]{sun2022communication}
Zhenyu Sun and Ermin Wei.
\newblock A communication-efficient algorithm with linear convergence for
  federated minimax learning.
\newblock {\em Advances in Neural Information Processing Systems},
  35:6060--6073, 2022.

\bibitem[WXX{\etalchar{+}}23]{wang2023bose}
Lun Wang, Yang Xu, Hongli Xu, Zhida Jiang, Min Chen, Wuyang Zhang, and Chen
  Qian.
\newblock {BOSE}: Block-wise federated learning in heterogeneous edge
  computing.
\newblock {\em IEEE/ACM Transactions on Networking}, 2023.

\bibitem[XHZH21]{xian2021faster}
Wenhan Xian, Feihu Huang, Yanfu Zhang, and Heng Huang.
\newblock A faster decentralized algorithm for nonconvex minimax problems.
\newblock {\em Advances in Neural Information Processing Systems},
  34:25865--25877, 2021.

\bibitem[Xu24]{xu2024decentralized}
Yangyang Xu.
\newblock Decentralized gradient descent maximization method for composite
  nonconvex strongly-concave minimax problems.
\newblock {\em SIAM Journal on Optimization}, 34(1):1006--1044, 2024.

\bibitem[YLZL22]{yang2022sagda}
Haibo Yang, Zhuqing Liu, Xin Zhang, and Jia Liu.
\newblock {SAGDA}: Achieving $\mathcal{O}(\epsilon^{-2})$ communication
  complexity in federated min-max learning.
\newblock {\em Advances in Neural Information Processing Systems},
  35:7142--7154, 2022.

\bibitem[ZLL{\etalchar{+}}23]{zhou2023fedpage}
Guangmeng Zhou, Qi~Li, Yang Liu, Yi~Zhao, Qi~Tan, Su~Yao, and Ke~Xu.
\newblock {FedPAGE}: Pruning adaptively toward global efficiency of
  heterogeneous federated learning.
\newblock {\em IEEE/ACM Transactions on Networking}, 2023.

\bibitem[ZY19]{zhang2019decentralized}
Jiaqi Zhang and Keyou You.
\newblock Decentralized stochastic gradient tracking for non-convex empirical
  risk minimization.
\newblock {\em arXiv preprint arXiv:1909.02712}, 2019.

\bibitem[ZYG{\etalchar{+}}21]{zhang2021complexity}
Siqi Zhang, Junchi Yang, Crist{\'o}bal Guzm{\'a}n, Negar Kiyavash, and Niao He.
\newblock The complexity of nonconvex-strongly-concave minimax optimization.
\newblock In {\em Uncertainty in Artificial Intelligence}, pages 482--492.
  PMLR, 2021.

\end{thebibliography}
